\documentclass{article}
\PassOptionsToPackage{numbers, compress}{natbib}
\usepackage[preprint]{neurips_2023}
\usepackage[utf8]{inputenc}
\usepackage{amsmath}
\usepackage{algorithm}
\usepackage[noend]{algpseudocode}
\usepackage{amsfonts}
\usepackage{amssymb}
\usepackage{bm}
\usepackage{natbib}
\usepackage{bbm}
\usepackage{textcomp}
\usepackage{pgfplots}
\usepackage{bbm}
\usepackage{graphicx}
\usepackage{wrapfig}
\usepackage[]{youngtab}
\usepackage{amsthm}
\usepackage{tikz}
\usepackage{tikz-cd}
\usetikzlibrary{automata, positioning}
\usepackage{comment}
\usepackage{stmaryrd}

\usepackage{theoremref}
\usepackage{mathrsfs}

\usepackage{color-edits}
\addauthor{sw}{blue}
\addauthor{rr}{red}
\addauthor{jw}{brown}
\addauthor{ar}{red}

\usepackage[english]{babel}
\usepackage[utf8]{inputenc}
\usepackage{fancyhdr}
\pgfplotsset{width=10cm,compat=1.9}

\renewcommand{\P}{\mathbb{P}}

\newcommand{\E}{\mathbb{E}}
\newcommand{\R}{\mathbb{R}}

\newcommand{\calF}{\mathcal{F}}
\newcommand{\calX}{\mathcal{X}}

\newcommand{\cV}{\mathcal{V}}

\newcommand{\Epsilon}{\mathcal{E}}

\newcommand{\id}{\mathrm{id}}

\newcommand{\wh}[1]{\widehat{#1}}
\newcommand{\wt}[1]{\widetilde{#1}}
\newcommand{\spn}{\mathrm{span}}

\newtheorem{theorem}{Theorem}
\newtheorem{lemma}{Lemma}

\newtheorem{corollary}{Corollary}
\newtheorem{fact}{Fact}
\newtheorem{ass}{Assumption}

\theoremstyle{definition}

\newtheorem{defacto}[fact]{Definition/Fact}

\theoremstyle{remark}

\usepackage{thm-restate}
\usepackage{hyperref}
\hypersetup{
	colorlinks=false,
	bookmarks=true,
	breaklinks=true,
	hidelinks=true,
	pdfpagemode=empty,
}
\usepackage{nameref}
\usepackage[noabbrev, capitalize, nameinlink]{cleveref}
\crefname{prop}{Proposition}{Propositions}
\crefname{rmk}{Remark}{Remarks}
\crefname{cor}{Corollary}{Corollaries}
\crefname{claim}{Claim}{Claims}
\crefname{lemma}{Lemma}{Lemmata}
\crefname{example}{Example}{Examples}
\crefname{corollary}{Corollary}{Corollaries}
\title{On the Sublinear Regret of GP-UCB}
\author{%
  Justin Whitehouse\\
  Carnegie Mellon University\\
  \texttt{jwhiteho@andrew.cmu.edu}\\
    \And
  Zhiwei Steven Wu\\
  Carnegie Mellon University\\
  \texttt{zstevenwu@cmu.edu}\\
  \AND
  Aaditya Ramdas\\
  Carnegie Mellon University \\
  \texttt{aramdas@cmu.edu} \\
}
\date{\today}

\begin{document}
\maketitle

\begin{abstract}
In the kernelized bandit problem, a learner aims to sequentially compute the optimum of a function lying in a reproducing kernel Hilbert space given only noisy evaluations at sequentially chosen points. In particular, the learner aims to minimize regret, which is a measure of the suboptimality of the choices made.
Arguably the most popular algorithm is the Gaussian Process Upper Confidence Bound (GP-UCB) algorithm, which involves acting based on a simple linear estimator of the unknown function.
Despite its popularity, existing analyses of GP-UCB give a suboptimal regret rate, which fails to be sublinear for many commonly used kernels such as the Mat\'ern kernel. This has led to a longstanding open question: are existing regret analyses for GP-UCB tight, or can bounds be improved by using more sophisticated analytical techniques?
In this work, we resolve this open question and show that GP-UCB enjoys nearly optimal regret. In particular, our results yield sublinear regret rates for the Mat\'ern kernel, improving over the state-of-the-art analyses and partially resolving a COLT open problem posed by Vakili et al. Our improvements rely on a key technical contribution --- regularizing kernel ridge estimators in proportion to the smoothness of the underlying kernel $k$. Applying this key idea together with a largely overlooked concentration result in separable Hilbert spaces (for which we provide an independent, simplified derivation), we are able to provide a tighter analysis of the GP-UCB algorithm.
\end{abstract}
\section{Introduction}
An essential problem in areas such as econometrics \citep{farias2022synthetically, hoffman2011portfolio}, medicine \citep{mansour2020bayesian, mate2020collapsing}, optimal control \citep{barto1994reinforcement, agrawal1995continuum}, and advertising \citep{li2010contextual} is to optimize an unknown function given \emph{bandit feedback}, in which  algorithms only get to observe the outcomes for the chosen actions. Due to the bandit feedback, there is a fundamental tradeoff between \textit{exploiting} what has been observed about the local behavior of the function and \textit{exploring} to learn more about the function's global behavior. There has been a long line of work on bandit learning that investigates this tradeoff across different settings, including multi-armed bandits \citep{slivkins2019introduction, lattimore2020bandit, whittle1980multi}, linear bandits \citep{abbasi2011improved, soare2014best}, and kernelized bandits \citep{chowdhury2017kernelized, shekhar2018gaussian, vakili2021information}.

%Many frameworks have been developed to address this fundamental tradeoff, with some of the most prominent being (A) multi-armed bandits \citep{slivkins2019introduction, lattimore2020bandit, whittle1980multi}, in which a learner must effectively identify which of $d$ actions offers the highest reward; (B) linear bandits \citep{abbasi2011improved, soare2014best}, in which a learner must sequentially learn some unknown slope vector in $\R^d$; and (C) kernelized bandits \citep{chowdhury2017kernelized, shekhar2018gaussian, vakili2021information}, in which a learner must optimize an unknown function with certain smoothness properties.

In this work, we focus on the kernelized bandit framework, which can be viewed as an extension of the well-studied linear bandit setting to an infinite-dimensional reproducing kernel Hilbert space (or RKHS) $(H, \langle \cdot, \cdot\rangle_H)$. In this problem, there is some unknown function $f^\ast : \calX \rightarrow \R$ of bounded norm in $H$, where $\calX \subset \R^d$ is a bounded set. In each round $t \in [T]$, the learner uses previous observations to select an action $X_t \in \calX$, and then observes feedback $Y_t := f^\ast(X_t) + \epsilon_t$, where $\epsilon_t$ is a zero-mean noise variable. The learner aims to minimize (with high probability) the regret at time $T$, which is defined as 
\[
R_T := \sum_{t = 1}^T f^\ast(x^\ast) - f^\ast(X_t)
\]
where $x^\ast := \arg\max_{x \in \calX}f^\ast(x)$. 
% Heuristically, the regret measures the fraction of time the learner selects a significantly suboptimal action. 
The goal is to develop simple, efficient algorithms for the kernelized bandit problem that minimize regret $R_T$. We make the following standard assumption. We also make assumptions on the underlying kernel $k$, which we discuss in Section~\ref{sec:back}.
\begin{ass}
\label{ass:regret}
We assume that (a) there is some constant $D > 0$ known to the learner such that $\|f^\ast\|_H \leq D$ and (b) for every $t \geq 1$, $\epsilon_t$ is $\sigma$-subGaussian conditioned on $\sigma(Y_{1:t - 1}, X_{1:t})$.

\end{ass}

Arguably the simplest algorithm for the kernelized bandit problem is GP-UCB (Gaussian process upper confidence bound) \citep{srinivas2009gaussian, chowdhury2017kernelized}. GP-UCB works by maintaining a kernel ridge regression estimator of the unknown function $f^\ast$ alongside a confidence ellipsoid, optimistically selecting in each round the action that provides the maximal payoff over all feasible functions. Not only is GP-UCB efficiently computable thanks to the kernel trick, but it also offers strong empirical guarantees \citep{chowdhury2017kernelized}. The only seeming deficit of GP-UCB is its regret guarantee, as existing analyses only show that, with high probability, $R_T = \wt{O}(\gamma_T\sqrt{T})$, where $\gamma_T$ is a kernel-dependent measure of complexity known as the maximum information gain~\citep{srinivas2009gaussian, cover1991information}. In contrast, more complicated, less computationally efficient algorithms such as SupKernelUCB \citep{valko2013finite, scarlett2017lower} have been shown to obtain regret bounds of $\wt{O}(\sqrt{\gamma_TT})$, improving over the analysis of GP-UCB by a multiplicative factor of $\sqrt{\gamma_T}$. This gap is stark as the bound $\wt{O}(\gamma_T\sqrt{T})$ fails, in general, to be sub-linear for the practically relevant Mat\'ern kernel, whereas $\wt{O}(\sqrt{\gamma_T T})$ is sublinear for \textit{any} kernel experiencing polynomial eigendecay \citep{vakili2021information}.

This discrepancy has prompted the development of many variants of GP-UCB that, while less computationally efficient, offer better, regret guarantees  in some situations \citep{janz2020bandit, shekhar2020multi, shekhar2022instance}. (See a detailed discussion of these algorithms along with other related work in Appendix~\ref{app:related}.) However, 
the following question remains an open problem in online learning \citep{vakili2021open}: are existing analyses of vanilla GP-UCB tight, or can an improved analysis show GP-UCB enjoys sublinear regret?

\subsection{Contributions}

In this work, we show that GP-UCB obtains almost optimal, sublinear regret for any kernel experiencing polynomial eigendecay. This, in particular, implies that GP-UCB obtains sublinear regret for the commonly used Mat\'ern family of kernels. We provide a brief roadmap of our paper below.
%We obtain our improvements by making two key observations. First, we observe that, by making a careful limiting argument, the self-normalized confidence bounds introduced for linear bandit problem by \citet{abbasi2011improved} hold unchanged in kernelized setting. These bounds, unlike the existing ones present in \citet{chowdhury2017kernelized}, are clean, interpretable, and have simple dependence on the regularization parameter. Following from this, our second observation is that the regularization parameter should not be treated as a constant, and that it should be instead selected in accordance with the smoothness of the underlying kernel. In sum, from these two points, we are able to construct a signficantly simplied, improved analysis for GP-UCB.

\begin{enumerate}
    \item In Section~\ref{sec:bound}, we provide background into self-normalized concentration in Hilbert spaces. In particular, in Theorem~\ref{thm:mixture_hilbert}, we provide an independent, simplified derivation of a bound due to \citet{abbasi2013online}, which concerns to self-normalized concentration of certain process in separable Hilbert spaces. This bound has been largely overlooked in the kernel bandit literature, so we draw attention to it in hopes it can be leveraged in solving further kernel-based learning problems. As opposed to the existing bound of \citet{chowdhury2017kernelized}, which involves employing a complicated ``double mixture'' argument, the bound we present follows directly from applying the well-studied finite-dimensional method of mixtures alongside a simple truncation argument \citep{de2004self, de2007pseudo, de2009theory, abbasi2011improved}. These bounds are clean and show simple dependence on the regularization parameter.
    \item In Section~\ref{sec:reg}, we use leverage the self-normalized concentration detailed in Theorem~\ref{thm:mixture_hilbert} to provide an improved regret analysis for GP-UCB. By carefully choosing regularization parameters based on the smoothness of the underlying kernel, we demonstrate that GP-UCB enjoys sublinear regret of $\wt{O}\left(T^{\frac{3 + \beta}{2 + 2\beta}}\right)$ for any kernel experiencing $(C, \beta)$-polynomial eigendecay. As a special case of this result, we obtain regret bounds of $\wt{O}\left(T^{\frac{\nu + 2d}{2\nu + 2d}}\right)$ for the commonly used Mat\'ern kernel with smoothness $\nu$ in dimension $d$. Our new analysis improves over existing state-of-the-art analysis for GP-UCB, which fails to guarantee sublinear regret in general for the Mat\'ern kernel family~\citep{chowdhury2017kernelized}, and thus partially resolves an open problem posed by \cite{vakili2021open} on the suboptimality of GP-UCB.
\end{enumerate}

In sum, our results show that GP-UCB, the go-to algorithm for the kernelized bandit problem, is nearly optimal, coming close to the algorithm-independent lower bounds of \citet{scarlett2017lower}. Our work thus can be seen as providing theoretical justification for the strong empirical performance of GP-UCB \citep{srinivas2009gaussian}. Perhaps the most important message of our work is the importance of careful regularization in online learning problems. While many existing bandit works treat the regularization parameter as a small, kernel-independent constant, we are able to obtain significant improvements by carefully selecting the regularization parameter. We hope our work will encourage others to pay close attention to the selection of regularization parameters in future works.
\section{Background and Problem Statement}
\label{sec:back}

\paragraph{Notation.} We briefly touch on basic definitions and notational conveniences that will be used throughout our work. If $a_1, \dots, a_t \in \R$, we let $a_{1:t} := (a_1, \dots, a_t)^\top$. Let $(H, \langle\cdot, \cdot\rangle_H)$ be a reproducing kernel Hilbert space associated with a kernel $k : \calX \times \calX \rightarrow \R$. We refer to the identity operator on $H$ as $\id_H$. This is distinct from the identity mapping on $\R^d$, which we will refer to as $I_d$. For elements $f, g \in H$, we define their outer product as $fg^\top := f\langle g, \cdot\rangle_H$ and inner product as $f^\top g := \langle f, g\rangle_H$. For any $t \geq 1$ and sequence of points $x_1, \dots, x_t \in \calX$ (which will typically be understood from context), let $\Phi_t := (k(\cdot, x_1), \dots, k(\cdot, x_t))^{\top}$. We can respectively define the Gram matrix $K_t : \R^t \rightarrow \R^t$ and covariance operator $V_t : H \rightarrow H$ as $K_t := (k(x_i, x_j))_{i, j \in [t]} = \Phi_t\Phi_t^\top$ and $V_t := \sum_{s = 1}^t k(\cdot, x_s)k(\cdot, x_s)^\top = \Phi_t^\top \Phi_t$. These two operators essentially encode the same information about the observed data points, the former being easier to work with when actually performing computations (by use of the well known kernel trick) and latter being easier to algebraically manipulate.

Suppose $A : H \rightarrow H$ is a Hermitian operator of finite rank; enumerate its non-zero eigenvalues as $\lambda_1(A), \dots, \lambda_k(A)$. We can define the Fredholm determinant of $I + A$ as $\det(I + A) := \prod_{m = 1}^k(1 + \lambda_i(A))$ \citep{lax2002functional}. For any $t \geq 1, \rho > 0$, and $x_1, \dots, x_t \in \calX$, one can check via a straightforward computation that $\det(I_t + \rho^{-1}K_t) = \det(\id_H + \rho^{-1}V_t)$, where $K_t$ and $V_t$ are the Gram matrix and covariance operator defined above. We, again, will use these two quantities interchangeably in the sequel, but will typically prefer the latter in our proofs.

If $(H, \langle\cdot, \cdot\rangle_H)$ is a (now general) separable Hilbert space and $(\varphi_n)_{n \geq 1}$ is an orthonormal basis for $H$, for any $N \geq 1$ we can define the orthogonal projection operator $\pi_N : H \rightarrow \spn\{\varphi_1, \dots, \varphi_N\} \subset H$ by $\pi_N f := \sum_{n = 1}^N \langle f, \varphi_n\rangle_H \varphi_n$. We can correspondingly the define the projection onto the remaining basis functions to be the map $\pi_N^\perp : H \rightarrow \spn\{\varphi_1, \dots, \varphi_N\}^\perp$ given by $\pi_N^\perp f := f - \pi_N f$. Lastly, if $A : H \rightarrow H$ is a symmetric, bounded linear operator, we let $\lambda_{\max}(A)$ denote the maximal eigenvalue of $A$, when such a value exists. In particular, $\lambda_{\max}(A)$ will exist whenever $A$ has a finite rank, as will typically be the case considered in this paper.

\paragraph{Basics on RKHSs.}
Let $\calX \subset \R^d$ be some domain. A \textit{kernel} is a positive semidefinite map $k : \calX \times \calX \rightarrow \R$ that is square-integrable, i.e. $\int_\calX\int_\calX |k(x, y)|^2dxdy < \infty $. Any kernel $k$ has an associated \textit{reproducing kernel Hilbert space} or \textit{RKHS} $(H, \langle\cdot, \cdot\rangle_H)$ containing the closed span of all partial kernel evaluations $k(\cdot, x), x \in \calX$. In particular, the inner product $\langle\cdot, \cdot\rangle_H$ on $H$ satisfies the reproducing relationship $f(x) = \langle f, k(\cdot, x)\rangle_H$ for all $x \in \calX$.

A kernel $k$ can be associated with a corresponding \textit{Hilbert-Schmidt operator}, which is the Hermitian operator $T_k : L^2(\calX) \rightarrow L^2(\calX)$ given by $(T_kf)(x) := \int_\calX f(y)k(x, y)dy$ for any $x \in \calX$. In short, $T_k$ can be thought of as ``smoothing out'' or ``mollifying'' a function $f$ according to the similarity metric induced by $k$. $T_k$ plays a key role in kernelized learning through  \textit{Mercer's Theorem}, which gives an explicit representation for $H$ in terms of the eigenvalues and eigenfunctions of $T_k$.

\begin{fact}[\textbf{Mercer's Theorem}]
\label{fact:mercer}
Let $(H, \langle \cdot, \cdot\rangle_H)$ be the RKHS associated with kernel $k$, and let $(\mu_n)_{n \geq 1}$ and $(\phi_n)_{n \geq 1}$ be the sequence of non-increasing eigenvalues and corresponding eigenfunctions for $T_k$. Let $(\varphi_n)_{n \geq 1}$ be the sequence of rescaled functions $\varphi_n := \sqrt{\mu}_n \phi_n$. Then,
\[
H = \left\{\sum_{n = 1}^\infty \theta_n \varphi_n : \sum_{n = 1}^\infty \theta_n^2 < \infty\right\},
\]
and $(\varphi_n)_{n \geq 1}$ forms an orthonormal basis for $(H, \langle \cdot, \cdot\rangle_H)$.

\end{fact}

We make the following assumption throughout the remainder of our work, which is standard and comes from \citet{vakili2021information}.

\begin{ass}[\textbf{Assumption on kernel $k$}]
\label{ass:kernel}
The kernel $k : \calX \times \calX \rightarrow \R$ satisfies (a) $|k(x, y)| \leq L$ for all $x, y \in \calX$, for some constant $L > 0$ and (b) $|\phi_n(x)| \leq B$ for all $x \in \calX$, for some $B > 0$.
\end{ass}

\paragraph{``Complexity'' of RKHS's.}
By the eigendecay of a kernel $k$, we really mean the rate of decay of the sequence of eigenvalues $(\mu_n)_{n \geq 1}$. In the literature, there are two common paradigms for studying the eigendecay of $k$: $(C_1, C_2, \beta)$-exponential eigendecay, under which $\forall n \geq 1, \mu_n \leq C_1\exp(-C_2n^\beta)$, and $(C, \beta)$-polynomial eigendecay, under which $\forall n \geq 1, \mu_n \leq Cn^{-\beta}$. For kernels experiencing exponential eigendecay, of which the squared exponential is the most important example, GP-UCB is known to be optimal up to poly-logarithmic factors. However, for kernels experiencing polynomial eigendecay, of which the Mat\'ern family is a common example, existing analyses of GP-UCB  fail to yield sublinear regret. It is this latter case we focus on in this work.

Given the above representation in Fact~\ref{fact:mercer}, it is clear that the eigendecay of the kernel $k$ governs the ``complexity'' or ``size'' of the RKHS $H$. We make this notion of complexity precise by discussing \textit{maximum information gain}, a sequential, kernel-dependent quantity governing concentration and hardness of learning in RKHS's \citep{cover1991information, srinivas2009gaussian, vakili2021information}.

Let $t \geq 1$ and $\rho > 0$ be arbitrary. The maximum information gain at time $t$ with regularization $\rho$ is the scalar $\gamma_t(\rho)$ given by
$$
\gamma_t(\rho) := \sup_{x_1, \dots, x_t \in \calX}\frac{1}{2}\log\det\left(\id_H + \rho^{-1}V_t\right) = \sup_{x_1, \dots, x_t \in \calX}\frac{1}{2}\log\det\left(I_t + \rho^{-1}K_t\right).
$$
Our presentation of maximum information gain differs from some previous works in that we encode the regularization parameter $\rho$ into our notation.
% \citep{vakili2021information, chowdhury2017kernelized, srinivas2009gaussian}. 
This inclusion is key for our results, as we obtain improvements by carefully selecting  $\rho$. \citet{vakili2021information} bound the rate of growth of $\gamma_t(\rho)$ in terms of the rate of eigendecay of the kernel $k$. We leverage the following fact in our main results.

\begin{fact}[\textbf{Corollary 1 in \citet{vakili2021information}}]
\label{fact:edecay}
Suppose that kernel $k$ satisfies Assumption~\ref{ass:kernel} and experiences $(C, \beta)$-polynomial eigendecay. Then, for any $t \geq 1$, we have
$$
\gamma_t(\rho) \leq \left(\left(\frac{CB^2t}{\rho}\right)^{1/\beta}\log^{-1/\beta}\left(1 + \frac{Lt}{\rho}\right) + 1\right)\log\left(1 + \frac{Lt}{\rho}\right).
$$

\end{fact}

We last define the practically relevant Mat\'ern kernel and discuss its eigendecay.

\begin{defacto}
\label{def:matern}
The Mat\'ern kernel with bandwidth $\sigma > 0$ and smoothness $\nu > 1/2$ is given by
\[
k_{\nu, \sigma}(x, y) := \frac{1}{\Gamma(\nu)2^{\nu - 1}}\left(\frac{\sqrt{2\nu} \|x - y\|_y}{\sigma}\right)^{\nu}B_\nu\left(\frac{\sqrt{2\nu}\|x-y\|_2}{\sigma}\right),
\]
where $\Gamma$ is the gamma function and $B_\nu$ is the modified Bessel function of the second kind. It is known that there is some constant $C > 0$ that may depend on $\sigma$ but not on $d$ or $\nu$ such that $k_{\nu, \sigma}$ experiences $\left(C, \frac{2\nu + d}{d}\right)$-eigendecay \citep{santin2016approximation, vakili2021information}.
\end{defacto}

\paragraph{Basics on martingale concentration:}

A filtration $(\calF_t)_{t \geq 0}$ is a sequence of $\sigma$-algebras satisfying $\calF_t \subset \calF_{t + 1}$ for all $t \geq 1$. If $(M_t)_{t \geq 0}$ is a $H$-valued process, we say $(M_t)_{t \geq 0}$ is a martingale with respect to $(\calF_t)_{t \geq 0}$ if (a) $(M_t)_{t \geq 0}$ is $(\calF_t)_{t \geq 0}$-adapted, and (b) $\E(M_t \mid \calF_{t - 1}) = M_{t - 1}$ for all $t \geq 1$. An $\R$-valued process is called a supermartingale if the equality in (b) is replaced with ``$\leq$'', i.e. supermartingales tend to decrease. 
Martingales are useful in many statistical applications due to their strong concentration of measure properties~\citep{howard2021time,waudby2020estimating}. The follow fact can be leveraged to provide time-uniform bounds on the growth of any non-negative supermartingale.

\begin{fact}[\textbf{Ville's Inequality}]\label{fact:ville}
Let $(M_t)_{t \geq 0}$ be a non-negative supermartingale with respect to some filtration. Suppose $\E M_0 = 1$. Then, for any $\delta \in (0, 1)$, we have
\[
\P\left(\exists t \geq 0 : M_t \geq \frac{1}{\delta}\right) \leq \delta.
\]
\end{fact}

See \citet{howard2020time} for a self-contained proof of Ville's inequality, and many applications.

If $\calF$ is a $\sigma$-algebra, and $\epsilon$ is an $\R$-valued random variable, we say $\epsilon$ is $\sigma$-subGaussian conditioned on $\calF$ if, for any $\lambda \in \R$, we have $\log\E\left(e^{\lambda \epsilon} \mid \calF\right) \leq \frac{\lambda^2\sigma^2}{2}$; in particular this condition implies that $\epsilon$ is mean zero. With this, we state the following result on self-normalized processes. To our understanding, the following result was first presented in some form as Example 4.2 of \citet{de2007pseudo} (in the setting of continuous local martingales), and can be derived leveraging the argument of Theorem 1 in \citet{de2009theory}. The exact form below was established (in the setting of discrete-time processes) in Theorem 1 of \citet{abbasi2011improved}, which is commonly leveraged to construct confidence ellipsoids in the linear bandit setting. 

\begin{fact}[\textbf{Example 4.2 from \citep{de2007pseudo}, Theorem 1 from \citep{abbasi2011improved}}]
\label{fact:mixture_finite}
Let $(\calF_t)_{t \geq 0}$ be a filtration, let $(X_t)_{t \geq 1}$ be an $(\calF_t)_{t \geq 0}$-predictable sequence in $\R^d$, and let $(\epsilon_t)_{t \geq 1}$ be a real-valued $(\calF_t)_{t \geq 1}$-adapted sequence such that conditional on $\calF_{t - 1}$, $\epsilon_t$ is mean zero and $\sigma$-subGaussian. Then, for any $\rho > 0$, the process $(M_t)_{t \geq 0}$ given by
\[
M_t := \frac{1}{\sqrt{\det(I_d + \rho^{-1}V_t)}}\exp\left\{\frac{1}{2}\left\|(\rho I_d + V_t)^{-1/2}S_t/\sigma\right\|_2^2\right\}
\]
is a non-negative supermartingale with respect to $(\calF_t)_{t \geq 0}$, where $S_t := \sum_{s = 1}^t \epsilon_s X_s$ and $V_t := \sum_{s = 1}^t X_sX_s^\top$. Consequently, by Fact~\ref{fact:ville}, for any confidence $\delta \in (0, 1)$, the following holds: with probability at least $1 - \delta$, simultaneously for all $t \geq 1$, we have
$$
\left\|(V_t + \rho I_d)^{-1/2}S_t\right\|_2 \leq \sigma\sqrt{2\log\left(\frac{1}{\delta}\sqrt{\det(I_d +\rho^{-1}V_t)}\right)}.
$$

\end{fact}

Note the simple dependence on the regularization parameter $\rho > 0$ in the above bound. While the regularization parameter $\rho$ doesn't prove important in regret analysis for linear bandits (where $\rho$ is treated as constant), the choice for $\rho$ will be critical in our setting. In the following section, we will discuss how Fact~\ref{fact:mixture_finite} can be extended to the setting of separable Hilbert spaces essentially verbatim (an observation first noticed by \citet{abbasi2013online}).

\section{A Remark on Self-Normalized Concentration in Hillbert Spaces}
\label{sec:bound}

We begin by discussing a key, self-normalized concentration inequality for martingales. We use this bound in the sequel to construct simpler, more flexible confidence ellipsoids than currently exist for GP-UCB. The bound we present (in Theorem~\ref{thm:mixture_hilbert} below) is, more or less, equivalent to Corollary 3.5 in the thesis of \citet{abbasi2013online}. Our result is mildly more general in the sense that it directly argues that a target mixture process is a nonnegative supermartingale. The result in \citet{abbasi2013online} is more general in the sense it allows the regularization (or shift) matrix to be non-diagonal. Either concentration result is sufficient for the regret bounds obtained in the sequel.

The aforementioned corollary in \citep{abbasi2013online}, quite surprisingly, has not been referenced in central works on the kernelized bandit problem, namely \citet{chowdhury2017kernelized} and \citet{vakili2021information, vakili2021open}. In fact, strictly weaker versions of the conclusion have been independently rediscovered in the context of kernel regression \citep{durand2018streaming}. We emphasize that this result of \citet{abbasi2013online} (and the surrounding technical conclusions) are very general and may allow for further improvements in problems related to kernelized learning. 

We now present Theorem~\ref{thm:mixture_hilbert}, providing a brief sketch and a full proof in Appendix~\ref{app:mix}. We believe our proof, which directly shows a target process is a nonnegative supermartingale, is of independent interest when compared to that of \citet{abbasi2013online} due to its simplicity. In particular, our proof follows from first principles, avoiding advanced topological notions of convergence (e.g. in the weak operator topology) and existence of certain Gaussian measures on separable Hilbert spaces, which were  heavily utilized in the proof of Corollary 3.5 in \citet{abbasi2013online}.

% In the case of RKHS's, only the following bound from \citet{chowdhury2017kernelized} is known.

\begin{comment}
\begin{theorem}[\textbf{Self-normalized concentration in Hilbert spaces}]
\jwcomment{
\label{thm:mixture_hilbert}
Let $(\calF_t)_{t \geq 0}$ be a filtration, $(f_t)_{t \geq 1}$ be an $(\calF_t)_{t \geq 0}$-adapted sequence in a separable Hilbert space $H$ such that (a) $\|f_t\|_H < \infty$ a.s. for all $t \geq 0$,  (b) $\E\left(f_t f_t^\top \mid \calF_{t - 1}\right)$ is trace class a.s., and (c) $\log\E\left(e^{\lambda \langle \nu, f_t \rangle_H} \mid \calF_{t - 1}\right) \leq \frac{\lambda^2\sigma^2}{2}\langle \nu, \E\left(f_t f_t^\top \mid \calF_{t - 1}\right)$, for all $\|\nu\| =  1$.  Defining $S_t := \sum_{s = 1}^t f_s$ and $V_t := \sum_{s = 1}^t \E\left(f_s f_s^\top \mid \calF_{t - 1}\right)$, we have that for any $\rho > 0$, the process $(M_t)_{t \geq 0}$ defined by
$$
M_t := \frac{1}{\sqrt{\det(\id_H + \rho^{-1}V_t)}}\exp\left\{\frac{1}{2}\left\|(\rho \id_H + V_t)^{-1/2}S_t/\sigma\right\|_H^2\right\}
$$
is a nonnegative supermartingale with respect to $(\calF_t)_{t \geq 0}$. Consequently, by Fact~\ref{fact:ville}, for any $\delta \in (0, 1)$, with probability at least $1 - \delta$, simultaneously for all $t \geq 1$, we have
\[
\left\|(V_t + \rho I_d)^{-1/2}S_t\right\|_H \leq \sigma\sqrt{2\log\left(\frac{1}{\delta}\sqrt{\det(\id_H +\rho^{-1}V_t)}\right)}.
\]
}
\end{theorem}

\end{comment}

\begin{theorem}[\textbf{Self-normalized concentration in Hilbert spaces}]
\label{thm:mixture_hilbert}
Let $(\calF_t)_{t \geq 0}$ be a filtration, $(f_t)_{t \geq 1}$ be an $(\calF_t)_{t \geq 0}$-predictable sequence in a separable Hilbert space\footnote{A space is separable if it has a countable, dense set. Separability is key, because it means we have a countable basis, whose first $N$ elements we project onto.} $H$ such that $\|f_t\|_H < \infty$ a.s. for all $t \geq 0$,  and $(\epsilon_t)_{t \geq 1}$ be an $(\calF_t)_{t \geq 1}$-adapted sequence in $\R$ such that conditioned on $\calF_{t - 1}$, $\epsilon_t$ is mean zero and $\sigma$-subGaussian.  Defining $S_t := \sum_{s = 1}^t \epsilon_s f_s$ and $V_t := \sum_{s = 1}^t f_s f_s^\top$, we have that for any $\rho > 0$, the process $(M_t)_{t \geq 0}$ defined by
$$
M_t := \frac{1}{\sqrt{\det(\id_H + \rho^{-1}V_t)}}\exp\left\{\frac{1}{2}\left\|(\rho \id_H + V_t)^{-1/2}S_t/\sigma\right\|_H^2\right\}
$$
is a nonnegative supermartingale with respect to $(\calF_t)_{t \geq 0}$. Consequently, by Fact~\ref{fact:ville}, for any $\delta \in (0, 1)$, with probability at least $1 - \delta$, simultaneously for all $t \geq 1$, we have
\[
\left\|(V_t + \rho I_d)^{-1/2}S_t\right\|_H \leq \sigma\sqrt{2\log\left(\frac{1}{\delta}\sqrt{\det(\id_H +\rho^{-1}V_t)}\right)}.
\]

\end{theorem}

We can summarize our independent proof in two simple steps. First, following from Fact~\ref{fact:mixture_finite}, the bound in Theorem~\ref{thm:mixture_hilbert} holds when we project $S_t$ and $V_t$ onto a finite number $N$ of coordinates, defining a ``truncated'' nonnegative supermartingale $M_t^{(N)}$. Secondly, we can make a limiting arugment, showing $M_t^{(N)}$ is ``essentially'' $M_t$ for large values of $N$.

\begin{proof}[Proof Sketch for Theorem~\ref{thm:mixture_hilbert}]
Let $(\varphi_n)_{n \geq 1}$ be an orthonormal basis for $H$, and, for any $N \geq 1$, let $\pi_N$ denote the projection operator onto $H_N := \spn\{\varphi_1, \dots, \varphi_N\}$. Note that the projected process $(\pi_N S_t)_{t \geq 1}$ is an $H$-valued martingale with respect to  $(\calF_t)_{t \geq 0}$. Further, note that the projected variance process $(\pi_N V_t \pi^\top_N)_{t \geq 0}$ satisfies
\[
\pi_N V_t \pi^\top_N = \sum_{s = 1}^t (\pi_N f_s)(\pi_N f_s)^\top.
\]
Since, for any $N \geq 1$, $H_N$ is a finite-dimensional Hilbert space, it follows from Lemma~\ref{lem:mixture_hilbert_finite} that the process $(M_t^{(N)})_{t \geq 0}$ given by
\[
M_t^{(N)} := \frac{1}{\sqrt{\det(\id_{H} + \rho^{-1}\pi_N V_t \pi_N^\top)}}\exp\left\{\frac{1}{2}\left\|(\rho \id_H +\pi_N V_t \pi_N^\top)^{-1/2}\pi_N S_t\right\|_{H}^2\right\},
\]
is a nonnegative supermartingale with respect to $(\calF_t)_{t \geq 0}$. One can check that, for any $t \geq 0$, $M_t^{(N)} \xrightarrow[N \rightarrow \infty]{} M_t$. Thus, Fatou's Lemma implies
\begin{align*}
    \E\left(M_t \mid \calF_{t - 1}\right) &= \E\left(\liminf_{N \rightarrow \infty}M_t^{(N)} \mid \calF_{t - 1}\right) \\
    &\leq \liminf_{N \rightarrow \infty}\E\left(M_t^{(N)} \mid \calF_{t - 1}\right) \\
    &\leq \liminf_{N \rightarrow \infty}M_{t - 1}^{(N)} \\
    &= M_{t - 1},
\end{align*}
which proves the first part of the claim. The second part of the claim follows from applying Fact~\ref{fact:ville} to the defined nonnegative supermartingale and rearranging. See Appendix~\ref{app:mix} for details.
\end{proof}

The following corollary specializes Theorem~\ref{thm:mixture_hilbert} (and thus Corollary 3.5 of \citet{abbasi2013online}) to the case where $H$ is a RKHS and $f_t = k(\cdot, X_t)$, for all $t \geq 1$. In this special case, we can reframe the above theorem in terms familiar Gram matrix $K_t$, assuming the quantity is invertible. While we prefer the simplicity and elegance of working directly in the RKHS $H$ in the sequel, the follow corollary allows us to present Theorem~\ref{thm:mixture_hilbert} in a way that is computationally tractable.

\begin{corollary}
\label{cor:gram_bd}
Let us assume the same setup as Theorem~\ref{thm:mixture_hilbert}, and additionally assume that (a) $(H, \langle\cdot, \cdot\rangle_H)$ is a RKHS associated with some kernel $k$, and (b) there is some $\calX$-valued $(\calF_t)_{t \geq 0}$-predictable process $(X_t)_{t \geq 1}$ such that $(f_t)_{t \geq 1} = (k(\cdot, X_t))_{t \geq 1}$. Then, for any $\rho > 0$ and $\delta \in (0, 1)$, we have that, with probability at least $1 - \delta$, simultaneously for all $t \geq 0$,
\begin{align*}
\left\|(V_t + \rho \id_H)^{-1/2}S_t\right\|_H \leq 
% \sigma\sqrt{2\log\left(\sqrt{\det(\id_H +\rho^{-1}V_t)}\right)} 
% = 
\sigma\sqrt{2\log\left(\sqrt{\frac{1}{\delta}\det(I_t +\rho^{-1}K_t)}\right)}.
\end{align*}
If, in addition, the Gram matrix $K_t = (k(X_i, X_j))_{i, j \in [t]}$ is invertible, we have the equality
\[
\|(I_t + \rho K_t^{-1})^{-1/2}\epsilon_{1:t}\|_2 = \|(\rho \id_H + V_t)^{-1/2}S_t\|_H .
\]
\end{corollary}

We prove Corollary~\ref{cor:gram_bd} in Appendix~\ref{app:mix}. With this reframing of Theorem~\ref{thm:mixture_hilbert}, we compare the concentration results of Theorem~\ref{thm:mixture_hilbert} (and thus \citet{abbasi2013online}) to  the following, commonly leveraged result from \citet{chowdhury2017kernelized}.

\begin{fact}[\textbf{Theorem 1 from \citet{chowdhury2017kernelized}}]
\label{fact:mixture_RKHS}
Assume the same setup as Fact~\ref{fact:mixture_finite}. Let $\eta > 0$ be arbitrary, and let $K_t := (k(X_i, X_j))_{i, j \in [t]}$ be the Gram matrix corresponding to observations made by time $t \geq 1$. Then, with probability at least $1 - \delta$, simultaneously for all $t \geq 1$, we have
$$
\left\|\left((K_t + \eta I_t)^{-1} + I_t\right)^{-1/2}\epsilon_{1:t}\right\|_2 \leq \sigma\sqrt{2\log\left(\frac{1}{\delta}\sqrt{\det\left((1 + \eta)I_t + K_t\right)}\right)}.
$$

\end{fact}

To make comparison with this bound clear, we parameterize the bounds in the above fact in terms of $\eta > 0$ instead of $\rho > 0$ to emphasize the following difference: both sides of the bound presented in Theorem~\ref{thm:mixture_hilbert} shrink as $\rho$ is increased, whereas both sides of the bound in Fact~\ref{fact:mixture_RKHS} increase as $\eta$ grows. Thus, increasing $\rho$ in Theorem~\ref{thm:mixture_hilbert} should be seen as decreasing $\eta$ in the bound of \citet{chowdhury2017kernelized}. The bounds in Corollary~\ref{cor:gram_bd} and Fact~\ref{fact:mixture_RKHS} coincide when $\rho = 1$ and $\eta \downarrow 0$ (per Lemma 1 in \citet{chowdhury2017kernelized}), but are otherwise not equivalent for other choices of $\rho$ and $\eta$. 

We believe Theorem~\ref{thm:mixture_hilbert} and Corollary 3.5 of \citet{abbasi2013online} to be signficantly more usable than the result of \citet{chowdhury2017kernelized} for several reasons. First, the aforementioned bounds \textit{directly} extend the method of mixtures (in particular, Fact~\ref{fact:mixture_finite}) to potentially infinite-dimensional Hilbert spaces. This similarity in form allows us to leverage existing analysis of \citet{abbasi2011improved} to prove our regret bounds, with only slight modifications. This is in contrast to the more cumbersome regret analysis that leverages Fact~\ref{fact:mixture_RKHS}, which is not only more difficult to follow, but also obtains inferior, sometimes super-linear regret guarantees.

Second, we note that Theorem~\ref{thm:mixture_hilbert} provides a bound that has a simple dependence on $\rho > 0$. In more detail, directly as a byproduct of the simplified bounds, Theorem~\ref{thm:reg} offers a regret bound that can readily be tuned in terms of $\rho$. Due to their use of a ``double mixture'' technique in proving Fact~\ref{fact:mixture_RKHS}, \citet{chowdhury2017kernelized} essentially wind up with a nested, doubly-regularized matrix $((K_t + \eta I_t)^{-1} + I_t)^{-1/2}$ with which they normalize the residuals $\epsilon_{1:t}$. In particular, this more complicated normalization make it difficult to understand how varying $\eta$ impacts regret guarantees, which we find to be essential for proving improved regret guarantees.

We note that the central bound discussed in this section \textit{does not} provide an improvement in dependence on maximum information gain in the sense hypothesized by \citet{vakili2021open}. In particular, the authors hypothesized the possibility of shaving a $\sqrt{\gamma_T}$ multiplicative factor off of self-normalized concentration inequalities in RKHS's. This was shown in a recent work (see \citet{lattimore2023lower}) to be impossible in general. Instead, Theorem~\ref{thm:mixture_hilbert} and Corollary 3.5 of \citet{abbasi2013online} give one access to a family of bounds parameterized by the regularization parameter $\rho > 0$. As will be seen in the sequel, by optimizing over this parameter, one can obtain significant improvements in regret.

\begin{comment}
Note that both sides of our bound shrink with $\rho$, approaching $0$ as $\rho$ approaches infinity. This is expected, as regularization serves to decrease the impact of noise (which is the term we are bounding) on the regression estimate, thus preventing overfitting. Since both sides of the bound in Fact~\ref{fact:mixture_RKHS} increase when $\rho$ grows, to make a fair comparison, we must consider the regime where $\rho$ tends toward 0. In this case, the bound reduces to
\begin{equation}
\label{ineq:bad_bd}
\|(K_t^{-1} + I_t)^{-1/2}\epsilon_{1:t}\|_2 \leq \sigma\sqrt{2\log\left(\frac{1}{\delta}\sqrt{\det(I_t + K_t)}\right)}.
\end{equation}
Note that this is precisely equivalent to Corollary~\ref{cor:gram_bd} when $\rho$ is set to 1. However, contrary to intuition, no setting of $\rho$ causes both sides of the bound in Fact~\ref{fact:mixture_RKHS} to tend towards zero. This ultimately prevents the use of appropriate regularization as a tool for minimizing regret, as in the sequel we will be selecting $\rho$ in proportion to the time horizon $T$, thus causing both sides of the bound to tend towards zero. We go in greater depth in Section~\ref{sec:reg} in showing how to appropriately regularize GP-UCB.
\end{comment}
%With the above comparisons in mind, we now present our proof of Theorem~\ref{thm:mixture_hilbert}. In our proof, we cite several technical lemmas that pedantic and thus uninteresting in nature. These can be found in Appendix~\ref{app:mix}

\section{An Improved Regret Analysis of GP-UCB}
\label{sec:reg}

In this section, we provide the second of our main contributions, which is an improved regret analysis for the GP-UCB algorithm. We provide a description of GP-UCB in Algorithm~\ref{alg:UCB}. While we state the algorithm directly in terms of quantities in the RKHS $H$, these quantities can be readily converted to those involving Gram matrices or Gaussian processes for those who prefer that perspective~\citep{chowdhury2017kernelized, williams2006gaussian}. 

As seen in Section~\ref{sec:bound}, by carefully extending the ``method of mixtures'' technique (originally by Robbins)  of \citet{abbasi2011improved, abbasi2013online} and \citet{de2004self, de2007pseudo} to Hilbert spaces, we can construct self-normalized concentration inequalities that have simple dependence on the regularization parameter $\rho$. These simplified bounds, in conjunction with information about the eigendecay of the kernel $k$ \citep{vakili2021information}, can be combined to carefully choose $\rho$ to obtain improved regret. We now present our main result.
%Going forward, we make the following typical assumption on the problem instance.

\begin{algorithm}

\caption{Gaussian Process Upper Confidence Bound (GP-UCB)}
\begin{algorithmic}
\Require Regularization parameter $\rho > 0$, norm bound $D$, confidence bounds $(U_t)_{t \geq 1}$, and time horizon $T$.
\State Set $V_0 := \rho \id_H$, $f_0 := 0$, $\Epsilon_0 := \left\{f \in H : \|f\|_H \leq D\right\}$
\For{$t = 1, \dots, T$}
    \State Let $(X_t, \wt{f}_t) := \arg\max_{x \in \calX, f \in \Epsilon_{t - 1}}(f, k(\cdot, x))_H$
    \State Play action $X_t$ and observe reward $Y_t := f^\ast(X_t) + \epsilon_t$
    \State Set $V_t := V_{t - 1} + k(\cdot, X_t)k(\cdot, X_t)^\top$ and $f_t := (V_t + \rho \id_H)^{-1}\Phi_t^\top Y_{1:t}$
    \State Set $\Epsilon_t := \left\{f \in H : \left\|(V_t + \rho\id_H)^{1/2}(f_{t} - f)\right\|_H \leq U_t\right\}$
\EndFor

\end{algorithmic}
\label{alg:UCB}

\end{algorithm}
\begin{comment}
\begin{ass}
\label{ass:regret}
We assume that (a) there is some constant $D > 0$ known to the learner such that $\|f^\ast\|_H \leq D$ and (b) for every $t \geq 1$, $\epsilon_t$ is $\sigma$-subGaussian conditioned on $\sigma(Y_{1:t - 1}, X_{1:t})$.

\end{ass}

With the above paragraph in mind, we now present the main result of our work.
\end{comment}
\begin{theorem}
\label{thm:reg}
Let $T > 0$ be a fixed time horizon, $\rho > 0$ a regularization parameter, and assume Assumptions~\ref{ass:kernel} and \ref{ass:regret} hold. Let $\delta \in (0, 1)$, and for $t \geq 1$ define 
\[ U_t := \sigma\sqrt{2\log\left(\frac{1}{\delta}\sqrt{\det(\id_H + \rho^{-1}V_t)}\right)} + \rho^{1/2}D.
\]
Then, with probability at least $1 - \delta$, the regret of Algorithm~\ref{alg:UCB} run with parameters $\rho, (U_t)_{t \geq 1}, D$ satisfies
\[
R_T = O\left(\gamma_T(\rho)\sqrt{T} + \sqrt{\rho\gamma_T(\rho)T}\right),
\]
where in the big-Oh notation above we treat $\delta, D, \sigma, B$, and $L$ as being held constant. If the kernel $k$ experiences $(C, \beta)$-polynomial eigendecay for some $C > 0$ and $\beta > 1$, taking $\rho = O(T^{\frac{1}{1 + \beta}})$ yields $ R_T = \widetilde{O}\left(T^{\frac{3 + \beta}{2 + 2\beta}}\right)$\footnote{The notation $\wt{O}$ suppresses multiplicative, poly-logarithmic factors in $T$}, which is always sub-linear in $T$.

\end{theorem}

While we present the above bound with a fixed time-horizon, it can be made anytime by carefully applying a standard doubling argument (see \citet{lattimore2020bandit}, for instance). We specialize the above theorem to the case of the  Mat\'ern kernel in the following corollary.

\begin{corollary}
\label{cor:matern}
Definition~\ref{def:matern} states that the Mat\'ern kernel with smoothness $\nu > 1/2$ in dimension $d$ experiences $(C, \frac{2\nu + d}{d})$-eigendecay, for some constnat $C > 0$. Thus, GP-UCB obtains a regret rate of $R_T = \wt{O}\left(T^{\frac{\nu + 2d}{2\nu + 2d}}\right)$.
\end{corollary}

We note that our regret analysis is the first to show that GP-UCB attains sublinear regret for general kernels experiencing polynomial eigendecay. Of particular import is that Corollary~\ref{cor:matern} of Theorem~\ref{thm:reg} yields the first analysis of GP-UCB that implies sublinear regret for the Mat\'ern kernel under general settings of ambient dimension $d$ and smoothness $\nu$. A recent result by \citet{janz2022sequential}, using a uniform lengthscale argument, demonstrates that GP-UCB obtains sublinear regret for the specific case of the Mat\'ern family when the parameter $\nu$ and dimension $d$ satisfy a uniform boundedness condition independent of scale. Our results are (a) more general, holding for \textit{any} kernel exhibiting polynomial eigendecay, (b) don't require checking uniform boundedness independent of scale condition, and (c) follow from a simple regularization based argument. In particular, the arguments of \citet{janz2022sequential} require advanced functional analytic and Fourier analytic machinery.

We note that our analysis does not obtain optimal regret, as the theoretically interesting but computationally cumbersome SupKernelUCB algorithm \citep{scarlett2017lower, valko2013finite} obtains a slightly improved regret bound of $\wt{O}\left(T^{\frac{\beta + 1}{2\beta}}\right)$ for $(C, \beta)$-polynomial eigendecay and $\wt{O}\left(T^{\frac{\nu + d}{2\nu + d}}\right)$ for the Mat\'ern kernel with smoothness $\nu$ in dimension $d$. Due to the aforementioned result of \citet{lattimore2023lower}, which shows that improved dependence on maximum information gain cannot be generally obtained in Hilbert space concentration, we believe further improvements on regret analysis for GP-UCB may not possible.

To wrap up this section, we provide a proof sketch for Theorem~\ref{thm:reg}. The entire proof, along with full statements and proofs of the technical lemmas, can be found in Appendix~\ref{app:reg}.

\begin{proof}[\textbf{Proof Sketch for Theorem~\ref{thm:reg}}]
Letting, for any $t \in [T]$, the ``instantaneous regret'' be defined as $r_t := f^\ast(x^\ast) - f^\ast(X_t)$, a standard argument yields that, with probability at least $1 - \delta$, simultaneously for all $t \in [T]$,
\[
r_t \leq 2 U_{t - 1}\left\|(\rho \id_H + V_{t - 1})^{-1/2}k(\cdot, X_t)\right\|_H.
\]
A further standard argument using Cauchy-Schwarz and an elliptical potential argument yields
\begin{align*}
R_T &= \sum_{t = 1}^T r_t \leq U_T\sqrt{2T\log\det(\id_H + \rho^{-1}V_T)} \\
&= \left(\sigma\sqrt{2\log\left(\frac{1}{\delta}\sqrt{\det(\id_H + \rho^{-1}V_T)}\right)} + \rho^{1/2}D\right)\sqrt{2T\log\det(\id_H + \rho^{-1}V_T)} \\
&\leq \left(\sigma\sqrt{2\log(1/\delta)} + \sigma\sqrt{2\gamma_T(\rho)} + \rho^{1/2}D\right)\sqrt{4T\gamma_T(\rho)} = O\left(\gamma_T(\rho)\sqrt{T} + \sqrt{\rho\gamma_T(\rho)T}\right),
\end{align*}
which proves the first part of the claim. If, additionally, $k$ experiences $(C, \beta)$-polynomial eigendecay, we know that $\gamma_T(\rho) = \wt{O}\left(\left(\frac{T}{\rho}\right)^{1/\beta}\right)$ by Fact~\ref{fact:edecay}. Setting $\rho := O(T^{\frac{1}{1 + \beta}})$ thus yields
\[
R_T = O\left(\gamma_T(\rho)\sqrt{T} + \sqrt{\rho\gamma_T(\rho)T}\right) = \wt{O}\left(T^{\frac{3 + \beta}{2 + 2\beta}}\right),
\]
proving the second part of the claim.

\end{proof}

\section{Conclusion}

In this work, we present an improved analysis for the GP-UCB algorithm in the kernelized bandit problem. We provide the first analysis showing that GP-UCB obtains sublinear regret when the underlying kernel $k$ experiences polynomial eigendecay, which in particular implies sublinear regret rates for the practically relevant Mat\'ern kernel. In particular, we show GP-UCB obtains regret $\wt{O}\left(T^{\frac{3 + \beta}{2 + 2\beta}}\right)$ when $k$ experiences $(C, \beta)$-polynomial eigendecay, and regret $\wt{O}\left(T^{\frac{\nu + 2d}{2\nu + 2d}}\right)$ for the Mat\'ern kernel with smoothness $\nu$ in dimension $d$.

Our contributions are twofold. First, we show the importance of finding the ``right'' concentration inequality for tackling problems in online learning --- in this case the correct bound being a self-normalized inequality originally due to \citet{abbasi2013online}. We provide an independent proof of a result equivalent to Corollary 3.5 of \citet{abbasi2013online} in Theorem~\ref{thm:mixture_hilbert}, and hope that our simplified, truncation-based analysis will make the result more accessible to researchers working on problems in kernelized learning. Second, we demonstrate the importance of regularization in the kernelized bandit problem. In particular, since the smoothness of the kernel $k$ governs the hardness of learning, by regularizing in proportion to the rate of eigendecay of $k$, one can obtain significantly improved regret bounds.

A shortcoming of our work is that, despite obtaining the first generally sublinear regret bounds for GP-UCB, our rates are not optimal. In particular, there are discretization-based algorithms, such as SupKernelUCB \citep{valko2013finite}, which obtain slightly better regret bounds of $\wt{O}\left(T^{\frac{1 + \beta}{2\beta}}\right)$ for $(C, \beta)$-polynomial eigendecay. We hypothesize that the vanilla GP-UCB algorithm, which involves constructing confidence ellipsoids directly in the RKHS $H$, cannot obtain this rate. 

The common line of reasoning \citep{vakili2021open} is that because the Lin-UCB (the equivalent algorithm in $\R^d$) obtains the optimal regret rate of $\wt{O}(d\sqrt{T})$ in the linear bandit problem setting, then GP-UCB should attain optimal regret as well. In the linear bandit setting, there is no subtlety between estimating the optimal action and unknown slope vector, as these are one and the same. In the kernel bandit setting, estimating the function and optimal action are not equivalent tasks. In particular, the former serves in essence as a nuisance parameter in estimating the latter: tight estimation of unknown function under the Hilbert space norm implies tight estimation of the optimal action, but not the other way around. Existing optimal algorithms are successful because they discretize the input domain, which has finite metric dimension \citep{shekhar2018gaussian}, and make no attempts to estimate the unknown function in RKHS norm. Since compact sets in RKHS's do not, in general, have finite metric dimension \citep{wainwright2019high}, this makes estimation of the unknown function a strictly more difficult task. In fact, recent work by \citet{lattimore2023lower} demonstrate that self-normalized concentration in RKHS's, in general, cannot exhibit improved dependence on maximum information gain. This further supports our hypothesis on the further unimprovability of the regret analysis of GP-UCB past the improvements made in this paper.
\section{Acknowledgements}
 AR acknowledges support from NSF DMS 1916320 and an ARL IoBT CRA grant. Research reported in this paper was sponsored in part by the DEVCOM Army Research Laboratory under Cooperative Agreement W911NF-17-2-0196 (ARL IoBT CRA). The views and conclusions contained in this document are those of the authors and should not be interpreted as representing the official policies, either expressed or implied, of the Army Research Laboratory or the U.S. Government. The U.S. Government is authorized to reproduce and distribute reprints for Government purposes notwithstanding any copyright notation herein.

ZSW and JW were supported in part by the NSF CNS2120667, a CyLab 2021 grant, a Google Faculty Research Award, and a Mozilla Research Grant.

JW acknowledges support from NSF GRFP grants DGE1745016 and DGE2140739.

We also would like to thank Xingyu Zhou and Johannes Kirschner for independently bringing to our attention the result from \citet{abbasi2013online} (Corollary 3.5) on self-normalized concentration in Hilbert spaces which is essentially equivalent to Theorem~\ref{thm:mixture_hilbert}. We have rewritten the paper in a way that emphasizes the importance of this result and provides proper attribution to the original author.

\bibliography{bib.bib}{}

\begin{thebibliography}{38}
\providecommand{\natexlab}[1]{#1}
\providecommand{\url}[1]{\texttt{#1}}
\expandafter\ifx\csname urlstyle\endcsname\relax
  \providecommand{\doi}[1]{doi: #1}\else
  \providecommand{\doi}{doi: \begingroup \urlstyle{rm}\Url}\fi

\bibitem[Abbasi-Yadkori(2013)]{abbasi2013online}
Yasin Abbasi-Yadkori.
\newblock Online learning for linearly parametrized control problems.
\newblock 2013.

\bibitem[Abbasi-Yadkori et~al.(2011)Abbasi-Yadkori, P{\'a}l, and
  Szepesv{\'a}ri]{abbasi2011improved}
Yasin Abbasi-Yadkori, D{\'a}vid P{\'a}l, and Csaba Szepesv{\'a}ri.
\newblock Improved algorithms for linear stochastic bandits.
\newblock \emph{Advances in Neural Information Processing Systems}, 24, 2011.

\bibitem[Agrawal(1995)]{agrawal1995continuum}
Rajeev Agrawal.
\newblock The continuum-armed bandit problem.
\newblock \emph{SIAM Journal on Control and Optimization}, 33\penalty0
  (6):\penalty0 1926--1951, 1995.

\bibitem[Barto(1994)]{barto1994reinforcement}
Andrew~G Barto.
\newblock Reinforcement learning control.
\newblock \emph{Current Opinion in Neurobiology}, 4\penalty0 (6):\penalty0
  888--893, 1994.

\bibitem[Chowdhury and Gopalan(2017)]{chowdhury2017kernelized}
Sayak~Ray Chowdhury and Aditya Gopalan.
\newblock On kernelized multi-armed bandits.
\newblock In \emph{International Conference on Machine Learning}, pages
  844--853. PMLR, 2017.

\bibitem[Cover and Thomas(1991)]{cover1991information}
Thomas~M Cover and Joy~A Thomas.
\newblock Information theory and statistics.
\newblock \emph{Elements of Information Theory}, 1\penalty0 (1):\penalty0
  279--335, 1991.

\bibitem[de~la Pe{\~n}a et~al.(2004)de~la Pe{\~n}a, Klass, and Lai]{de2004self}
Victor de~la Pe{\~n}a, Michael~J Klass, and Tze~Leung Lai.
\newblock Self-normalized processes: Exponential inequalities, moment bounds
  and iterated logarithm laws.
\newblock \emph{The Annals of Probability}, 32, 07 2004.
\newblock \doi{10.1214/009117904000000397}.

\bibitem[de~la Pe{\~n}a et~al.(2007)de~la Pe{\~n}a, Klass, and
  Lai]{de2007pseudo}
Victor de~la Pe{\~n}a, Michael~J Klass, and Tze~Leung Lai.
\newblock Pseudo-maximization and self-normalized processes.
\newblock \emph{Probability Surveys Vol}, 4:\penalty0 172--192, 09 2007.
\newblock \doi{10.1214/07-PS119}.

\bibitem[de~la Pe{\~n}a et~al.(2009)de~la Pe{\~n}a, Klass, and
  Lai]{de2009theory}
Victor~H de~la Pe{\~n}a, Michael~J Klass, and Tze~Leung Lai.
\newblock Theory and applications of multivariate self-normalized processes.
\newblock \emph{Stochastic Processes and their Applications}, 119\penalty0
  (12):\penalty0 4210--4227, 2009.

\bibitem[Durand et~al.(2018)Durand, Maillard, and Pineau]{durand2018streaming}
Audrey Durand, Odalric-Ambrym Maillard, and Joelle Pineau.
\newblock Streaming kernel regression with provably adaptive mean, variance,
  and regularization.
\newblock \emph{The Journal of Machine Learning Research}, 19\penalty0
  (1):\penalty0 650--683, 2018.

\bibitem[Durrett(2019)]{durrett2019probability}
Rick Durrett.
\newblock \emph{Probability: theory and examples}, volume~49.
\newblock Cambridge university press, 2019.

\bibitem[Farias et~al.(2022)Farias, Moallemi, Peng, and
  Zheng]{farias2022synthetically}
Vivek Farias, Ciamac Moallemi, Tianyi Peng, and Andrew Zheng.
\newblock Synthetically controlled bandits.
\newblock \emph{arXiv preprint arXiv:2202.07079}, 2022.

\bibitem[Hoffman et~al.(2011)Hoffman, Brochu, and
  De~Freitas]{hoffman2011portfolio}
Matthew Hoffman, Eric Brochu, and Nando De~Freitas.
\newblock Portfolio allocation for {B}ayesian optimization.
\newblock In \emph{UAI}, pages 327--336, 2011.

\bibitem[Howard et~al.(2020)Howard, Ramdas, McAuliffe, and
  Sekhon]{howard2020time}
Steven~R Howard, Aaditya Ramdas, Jon McAuliffe, and Jasjeet Sekhon.
\newblock Time-uniform {C}hernoff bounds via nonnegative supermartingales.
\newblock \emph{Probability Surveys}, 17:\penalty0 257--317, 2020.

\bibitem[Howard et~al.(2021)Howard, Ramdas, McAuliffe, and
  Sekhon]{howard2021time}
Steven~R Howard, Aaditya Ramdas, Jon McAuliffe, and Jasjeet Sekhon.
\newblock Time-uniform, nonparametric, nonasymptotic confidence sequences.
\newblock \emph{The Annals of Statistics}, 49\penalty0 (2), 2021.

\bibitem[Janz(2022)]{janz2022sequential}
David Janz.
\newblock \emph{Sequential decision making with feature-linear models}.
\newblock PhD thesis, 2022.

\bibitem[Janz et~al.(2020)Janz, Burt, and Gonz{\'a}lez]{janz2020bandit}
David Janz, David Burt, and Javier Gonz{\'a}lez.
\newblock Bandit optimisation of functions in the mat{\'e}rn kernel {RKHS}.
\newblock In \emph{International Conference on Artificial Intelligence and
  Statistics}, pages 2486--2495. PMLR, 2020.

\bibitem[Lattimore(2023)]{lattimore2023lower}
Tor Lattimore.
\newblock A lower bound for linear and kernel regression with adaptive
  covariates.
\newblock In \emph{The Thirty Sixth Annual Conference on Learning Theory},
  pages 2095--2113. PMLR, 2023.

\bibitem[Lattimore and Szepesv{\'a}ri(2020)]{lattimore2020bandit}
Tor Lattimore and Csaba Szepesv{\'a}ri.
\newblock \emph{Bandit Algorithms}.
\newblock Cambridge University Press, 2020.

\bibitem[Lax(2002)]{lax2002functional}
Peter~D Lax.
\newblock \emph{Functional Analysis}, volume~55.
\newblock John Wiley \& Sons, 2002.

\bibitem[Li et~al.(2010)Li, Chu, Langford, and Schapire]{li2010contextual}
Lihong Li, Wei Chu, John Langford, and Robert~E Schapire.
\newblock A contextual-bandit approach to personalized news article
  recommendation.
\newblock In \emph{Proceedings of the 19th International Conference on World
  Wide web}, pages 661--670, 2010.

\bibitem[Mansour et~al.(2020)Mansour, Slivkins, and
  Syrgkanis]{mansour2020bayesian}
Yishay Mansour, Aleksandrs Slivkins, and Vasilis Syrgkanis.
\newblock Bayesian incentive-compatible bandit exploration.
\newblock \emph{Operations Research}, 68\penalty0 (4):\penalty0 1132--1161,
  2020.

\bibitem[Mate et~al.(2020)Mate, Killian, Xu, Perrault, and
  Tambe]{mate2020collapsing}
Aditya Mate, Jackson Killian, Haifeng Xu, Andrew Perrault, and Milind Tambe.
\newblock Collapsing bandits and their application to public health
  intervention.
\newblock \emph{Advances in Neural Information Processing Systems},
  33:\penalty0 15639--15650, 2020.

\bibitem[Santin and Schaback(2016)]{santin2016approximation}
Gabriele Santin and Robert Schaback.
\newblock Approximation of eigenfunctions in kernel-based spaces.
\newblock \emph{Advances in Computational Mathematics}, 42\penalty0
  (4):\penalty0 973--993, 2016.

\bibitem[Scarlett et~al.(2017)Scarlett, Bogunovic, and
  Cevher]{scarlett2017lower}
Jonathan Scarlett, Ilija Bogunovic, and Volkan Cevher.
\newblock Lower bounds on regret for noisy {G}aussian process bandit
  optimization.
\newblock In \emph{Conference on Learning Theory}, pages 1723--1742. PMLR,
  2017.

\bibitem[Shekhar and Javidi(2018)]{shekhar2018gaussian}
Shubhanshu Shekhar and Tara Javidi.
\newblock Gaussian process bandits with adaptive discretization.
\newblock \emph{Electronic Journal of Statistics}, 12\penalty0 (2):\penalty0
  3829 -- 3874, 2018.

\bibitem[Shekhar and Javidi(2020)]{shekhar2020multi}
Shubhanshu Shekhar and Tara Javidi.
\newblock Multi-scale zero-order optimization of smooth functions in an {RKHS}.
\newblock \emph{2022 IEEE International Symposium on Information Theory
  (ISIT)}, pages 288--293, 2020.

\bibitem[Shekhar and Javidi(2022)]{shekhar2022instance}
Shubhanshu Shekhar and Tara Javidi.
\newblock Instance dependent regret analysis of kernelized bandits.
\newblock In \emph{International Conference on Machine Learning}, pages
  19747--19772. PMLR, 2022.

\bibitem[Slivkins et~al.(2019)]{slivkins2019introduction}
Aleksandrs Slivkins et~al.
\newblock Introduction to multi-armed bandits.
\newblock \emph{Foundations and Trends{\textregistered} in Machine Learning},
  12\penalty0 (1-2):\penalty0 1--286, 2019.

\bibitem[Soare et~al.(2014)Soare, Lazaric, and Munos]{soare2014best}
Marta Soare, Alessandro Lazaric, and R{\'e}mi Munos.
\newblock Best-arm identification in linear bandits.
\newblock \emph{Advances in Neural Information Processing Systems}, 27, 2014.

\bibitem[Srinivas et~al.(2009)Srinivas, Krause, Kakade, and
  Seeger]{srinivas2009gaussian}
Niranjan Srinivas, Andreas Krause, Sham~M Kakade, and Matthias Seeger.
\newblock Gaussian process optimization in the bandit setting: No regret and
  experimental design.
\newblock \emph{International Conference on Machine Learning}, 2009.

\bibitem[Vakili et~al.(2021{\natexlab{a}})Vakili, Khezeli, and
  Picheny]{vakili2021information}
Sattar Vakili, Kia Khezeli, and Victor Picheny.
\newblock On information gain and regret bounds in {G}aussian process bandits.
\newblock In \emph{International Conference on Artificial Intelligence and
  Statistics}, pages 82--90. PMLR, 2021{\natexlab{a}}.

\bibitem[Vakili et~al.(2021{\natexlab{b}})Vakili, Scarlett, and
  Javidi]{vakili2021open}
Sattar Vakili, Jonathan Scarlett, and Tara Javidi.
\newblock Open problem: Tight online confidence intervals for {RKHS} elements.
\newblock In \emph{Conference on Learning Theory}, pages 4647--4652. PMLR,
  2021{\natexlab{b}}.

\bibitem[Valko et~al.(2013)Valko, Korda, Munos, Flaounas, and
  Cristianini]{valko2013finite}
Michal Valko, Nathaniel Korda, R{\'e}mi Munos, Ilias Flaounas, and Nelo
  Cristianini.
\newblock Finite-time analysis of kernelised contextual bandits.
\newblock \emph{Proceedings of the 29th Conference on Uncertainty in Artificial
  Intelligence}, 2013.

\bibitem[Wainwright(2019)]{wainwright2019high}
Martin~J Wainwright.
\newblock \emph{High-dimensional Statistics: A Non-asymptotic Viewpoint},
  volume~48.
\newblock Cambridge university press, 2019.

\bibitem[Waudby-Smith and Ramdas(2023)]{waudby2020estimating}
Ian Waudby-Smith and Aaditya Ramdas.
\newblock Estimating means of bounded random variables by betting.
\newblock \emph{Journal of the Royal Statistical Society, Series B}, 2023.

\bibitem[Whittle(1980)]{whittle1980multi}
Peter Whittle.
\newblock Multi-armed bandits and the {G}ittins index.
\newblock \emph{Journal of the Royal Statistical Society: Series B
  (Methodological)}, 42\penalty0 (2):\penalty0 143--149, 1980.

\bibitem[Williams and Rasmussen(2006)]{williams2006gaussian}
Christopher~KI Williams and Carl~Edward Rasmussen.
\newblock \emph{Gaussian processes for machine learning}, volume~2.
\newblock MIT press Cambridge, MA, 2006.

\end{thebibliography}
\bibliographystyle{plainnat}

\appendix
\newpage

\section{Related Work}
\label{app:related}
The kernelized bandit problem was first studied by \citet{srinivas2009gaussian}, who introduce the GP-UCB algorithm and characterize its regret in both the Bayesian and Frequentist setting. While the authors demonstrate that GP-UCB obtains sublinear regret in the Bayesian setting for the commonly used kernels, their bounds fail to be sublinear in general in the frequentist setting for the Mat\'ern kernel, one of the most popular kernel choices in practice. \citet{chowdhury2017kernelized} further study the performance of GP-UCB in the frequentist setting. In particular, by leveraging a martingale-based ``double mixture'' argument, the authors are able to significantly simplify the confidence bounds presented in \citet{srinivas2009gaussian}. Unfortunately, the arguments introduced by \citet{chowdhury2017kernelized} did not improve regret bounds beyond logarithmic factors, and thus GP-UCB continued to fail to obtain sublinear regret for certain kernels in their work. Lastly, \citet{janz2022sequential} are able to obtain sublinear regret guarantees for certain parameter settings of the Mat\'ern kernel --- in particular in settings where the eigenfunctions of the Hilbert-Schmidt operator associated with the kernel are uniformly bounded independent of scale (Definition 28 in the cited work). 

There are many other algorithms that have been created for kernelized bandits. \citet{janz2020bandit} introduce an algorithm specific to the Mat\'ern kernel that obtains significantly improved regret over GP-UCB. This algorithm adaptively partitions the input domain into small hypercubes and running an instance of GP-UCB in each element of the discretized domain. \citet{shekhar2022instance} introduce an algorithm called LP-GP-UCB, which augments the GP-UCB estimator with local polynomial corrections. While in the worst case this algorithm recovers the regret bound of \citet{chowdhury2017kernelized}, if additional information is known about the unknown function $f^\ast$ (e.g. it is Holder continuous), it can provide improved regret guarantees. Perhaps the most important non-GP-UCB algorithm in the literature is the SupKernel algorithm introduced by \citet{valko2013finite}, which discretizes the input domain and successively eliminates actions from play. This algorithm is signficant because, despite its complicated nature, it obtains regret rates that match known lower bounds provided by \citet{scarlett2017lower} up to logarithmic factors.

Intimately tied to the kernelized bandit problem is the information-theoretic quantity of maximum information gain~\citep{cover1991information, srinivas2009gaussian}, which is a sequential, kernel-specific measure of hardness of learning. Almost all preceding algorithms provide regret bounds in terms of the max information gain.
% \citep{srinivas2009gaussian, chowdhury2017kernelized, valko2013finite, janz2020bandit}. 
Of particular import for our paper is the work of \citet{vakili2021information}. In this work, the authors use a truncation argument to upper bound the maximum information gain of kernels in terms of their eigendecay. We directly employ these bounds in our improved analysis of GP-UCB. The max-information gain bounds presented in \citet{vakili2021information} can be coupled with the regret analysis in \citet{chowdhury2017kernelized} to yield a regret bound of $\wt{O}\left(T^{\frac{\nu + 3d/2}{2\nu + d}}\right)$ in the case of the Mat\'ern kernel with smoothness $\nu$ in dimension $d$. In particular, when $\nu \leq \frac{d}{2}$, this regret bound fails to be sublinear. In practical setting, $d$ is viewed as large and $\nu$ is taken to be $3/2$ or $5/2$, making these bounds vacuous~\citep{shekhar2018gaussian, williams2006gaussian} The regret bounds in this paper are sublinear for \textit{any} selection of smoothness $\nu > \frac{1}{2}$ and $d \geq 1$. Moreover, a simple computation yields that our regret bounds strictly improve over (in terms of $d$ and $\nu$) those implied by \citet{vakili2021information}.

Last, we touch upon the topic of self-normalized concentration, which is an integral tool for constructing confidence bounds in UCB-like algorithms. Heuristically, self-normalized aims to sequentially control the growth of processes that have been rescaled by their variance to look, roughly speaking, normally (or subGaussian) distributed. The prototypical example of self-normalized concentration in the bandit literature comes from \citet{abbasi2011improved}, wherein the authors use a well known technique called the ``method of mixtures'' to construct confidence ellipsoids for finite dimensional online regression estimates. The concentration result in the aforementioned work is a specialization of results in \citet{de2004self}, which provide self-normalized concentration for a wide variety of martingale-related processes, several of which have been recently improved~\citep{howard2020time}. In a work that is largely overlooked in the kernel bandit community, \citet{abbasi2013online} extend their concentration result from \citet{abbasi2011improved} to separable Hilbert spaces by using advanced functional analytic machinery. The bound we present in this work is equivalent to the aforementioned bound in separable Hilbert spaces --- we provide an independent, simpler proof that avoids needing advanced tools from functional analysis. Perhaps the best-known result on concentration in Hilbert spaces is that of \citet{chowdhury2017kernelized}, who extend the results of \citet{abbasi2011improved} to the kernel setting using a ``double mixture'' technique, allowing them to construct self-normalized concentration inequalities for infinite-dimensional processes in RKHS's. This bound has historically been used in analyzing kernel bandit algorithms, although as we show in this work the bound of \citet{abbasi2013online} (which we independently derive in Theorem~\ref{thm:mixture_hilbert}) is perhaps better suited for online kernelized learning problems.
\section{Technical Lemmas for Theorem~\ref{thm:mixture_hilbert}}
\label{app:mix}
In this appendix, prove Theorem~\ref{thm:mixture_hilbert} along with several corresponding technical lemmas. While many of the following results are intuitively true, we provide their proofs in full rigor, as there can be subtleties when working in infinite-dimensional spaces. Throughout, we assume that the subGaussian noise parameter is $\sigma = 1$. The general case can readily be recovered by considering the rescaled process $(S_t/\sigma)_{t \geq 0}$.

The first lemma we present is a restriction of Theorem~\ref{thm:mixture_hilbert} to the case where the underlying Hilbert space $(H, \langle \cdot, \cdot\rangle_H)$ is finite dimensional, say of dimension $N$. In this setting, the result essentially follows immediately from Fact~\ref{fact:mixture_finite}. All we need to do is construct a natural isometric isomorphism between the spaces $H$ and $\R^N$, and then argue that applying such a mapping doesn't alter the norm of the self-normalized process. 

\begin{lemma}
\label{lem:mixture_hilbert_finite}
Theorem~\ref{thm:mixture_hilbert} holds if we additionally assume that $H$ is finite dimensional, i.e. if there exists $N \geq 1$ and orthonormal functions $\varphi_1, \dots, \varphi_N$ such that
$$
H := \spn\left\{\varphi_1, \dots, \varphi_N\right\}.
$$
\end{lemma}

\begin{proof}
Let $\tau : H \rightarrow \R^N$ be the map that takes a function $f = \sum_{n = 1}^N\theta_n\varphi_n \in H$ to its natural embedding $\tau f := (\theta_1, \dots, \theta_N)^\top \in \R^N$. Not only is the map $\tau$ an isomorphism between $H$ and $\R^N$, but it is also an isometry, i.e. $\|f\|_H = \|\tau f\|_2$ for all $f \in H$. Further, $\tau$ satisfies the relation $\tau^\top = \tau^{-1}$.

Define the ``hatted'' processes $(\wh{S}_t)_{t \geq 1}$ and $(\wh{V}_t)_{t \geq 1}$, which take values in $\R^N$ and $\R^{N \times N}$ respectively as
\[
\wh{S}_t = \sum_{s = 1}^t \epsilon_s \tau k(\cdot, X_s) \qquad \text{and} \qquad \wh{V}_t = \sum_{s = 1}^t (\tau k(\cdot, X_s))(\tau k(\cdot, X_s))^\top.
\]
It is not hard to see that, by the linearity of $\tau$, that for any $t \geq 1$, we have $\wh{S}_t = \tau S_t$ and $\wh{V}_t = \tau V_t \tau^\top$. We observe  that (a) $(\wh{V}_t + \rho I_N)^{-1/2} = \tau(V_t + \rho\id_H)^{-1/2}\tau^{\top}$ and (b) that the eigenvalues of $\wh{V}_t$ are exactly those of $V_t$.

Since the processes $(\wh{S}_t)_{t \geq 1}$ and $(\wh{V}_t)_{t \geq 1}$ satisfy the assumptions of Theorem~\ref{fact:mixture_finite}, we see that the process $(M_t)_{t \geq 0}$ given by
\[
M_t := \frac{1}{\sqrt{\det(I_N + \rho^{-1}\wh{V}_t)}}\exp\left\{\frac{1}{2}\left\|(\rho I_N + \wh{V}_t)^{-1/2}\wh{S}_t\right\|_2^2\right\}
\]
is a non-negative supermartingale with respect to $(\calF_t)_{t \geq 0}$.
From observation (a), the fact $\tau$ is an isometry, and the fact $\tau^\top = \tau^{-1}$, it follows that
\begin{align*}
\left\|(\wh{V}_t + \rho I_N)^{-1/2}\wh{S}_t\right\|_2 &= \left\|\tau(V_t + \rho\id_H)^{-1/2}\tau^{\top}\tau S_t\right\|_2 \\
&= \left\|(V_t + \rho\id_H)^{-1/2}\tau^{-1}\tau S_t\right\|_H \\
&= \left\|(V_t + \rho \id_H)^{-1/2}S_t\right\|_H.
\end{align*}
Further, observation (b) implies that
\[
\det(I_N + \rho \wh{V}_t) = \det(\id_H + \rho V_t).
\]
Substituting these identities into the definition of $(M_t)_{t \geq 0}$ yields the desired result, i.e. that
\[
M_t = \frac{1}{\sqrt{\det(\id_H +\rho^{-1}V_t)}}\exp\left\{\frac{1}{2}\left\|(V_t + \rho I_d)^{-1/2}S_t\right\|_H^2\right\}.
\]
is a non-negative supermartingale with respect to $(\calF_t)_{t \geq 0}$. The remainder of the result follows from applying Ville's Inequality (Fact~\ref{fact:ville}) and rearranging.

\end{proof}

We can prove Theorem~\ref{thm:mixture_hilbert} by truncating the Hilbert space $H$ onto the first $N$ components, applying Lemma~\ref{lem:mixture_hilbert_finite} to the ``truncated'' processes $(\pi_N S_t)_{t \geq 0}$ and $(\pi_N V_t \pi_N)_{t \geq 0}$ to construct a relevant, non-negative supermartingale $M_t^{(N)}$, and then show that the error from truncation in this non-negative supermartingale tends towards zero as $N$ grows large. The following two technical lemmas are useful in showing that this latter truncation tends towards zero.

\begin{lemma}
\label{lem:op_norm_convergence}
For any $t \geq 1$, let $V_t$ be as in the statement of Theorem~\ref{thm:mixture_hilbert}, and let $\pi_N$ be as in Section~\ref{sec:back}. Then, we have
$$
\pi_N V_t \pi_N \xrightarrow[N \rightarrow \infty]{} V_t,
$$
where the above convergence holds under the operator norm on $H$.
\end{lemma}
\begin{comment}
\jwcomment{
\begin{proof}
    Since $\Delta V_t := V_t - V_{t - 1} = \E\left(f_t f_t^\top \mid \calF_{t - 1}\right)$ is of trace class with probability 1 (by assumption (b)), there 
\end{proof}
}
\end{comment}
\begin{proof}
    Fix $\epsilon>0$, $t \geq 1$, and for $s \in [t]$, let us write $f_s = \sum_{n = 1}^\infty\theta_n(s) \varphi_n$. Since we have assumed $\|f_t\|_H < \infty$ for all $t \geq 1$, there exists some  $N_t < \infty$  such that, for all $s \in [t]$, $\|\pi_{N_t}^\perp f_s\|_H^2 = \sum_{n = N_t + 1}^\infty \theta_n(s)^2 < \frac{\epsilon}{2t}$. We also have, for any $s \in [t]$ and $N \geq 1$,  that $f_s$ is an eigenfunction of $f_sf_s^\top\pi_N^\perp = f_s\langle f_s, \pi_N^\perp(\cdot)\rangle_H$ with corresponding (unique) eigenvalue $\|f_sf_s^\top\pi_N^\perp\|_{op} = \lambda_{\max}(f_sf_s^\top \pi_N^\perp)  = \|\pi_N^\perp f_s\|_H^2 = \sum_{n = N + 1}^\infty \theta_n(s)^2$. Observe that, as an orthogonal projection operator, $\pi_N$ is self-adjoint, i.e. $\pi_N = \pi_N^\top$. With this information, we see that, for $N \geq N_t$, we have
    \begin{align*}
    \left\|\pi_N V_t \pi_N - V_t\right\|_{op} &\leq \sum_{s = 1}^t\left\|\pi_N f_s f_s^\top \pi_N - f_sf_s^\top\right\|_{op} \\
    &= \sum_{s = 1}^t \left\|\pi_N f_s f_s^\top \pi_N - \pi_N f_s f_s^\top + \pi_N f_s f_s^\top - f_s f_s^\top\right\|_{op}\\
    &\leq \sum_{s = 1}^t \left\|\pi_N f_s f_s^\top \pi_N - \pi_N f_s f_s^\top \right\|_{op} + \left\|\pi_N f_s f_s^\top - f_s f_s^\top\right\|_{op} \\
    &\leq \sum_{s = 1}^t \left\|\pi_N\right\|_{op}\left\|f_s f_s^\top \pi_N - f_s f_s^\top \right\|_{op} + \left\|\pi_N f_s f_s^\top - f_s f_s^\top\right\|_{op} \\
    &= \sum_{s = 1}^t 2\left\|f_s f_s^\top\pi_N^\perp\right\|_{op} = \sum_{s = 1}^t 2\|\pi_N^\perp f_s\|_H^2 < \epsilon.
    \end{align*}
    Since $\epsilon > 0$ was arbitrary, we have shown the desired result.
\end{proof}

\begin{lemma}
\label{lem:det_convergence}
For any $t \geq 1$, let $V_t$ be as in Theorem~\ref{thm:mixture_hilbert}, $\rho > 0$ arbitrary, and $\pi_N$ as in Section~\ref{sec:back}. Then, we have 
\[
\det(\id_H + \rho^{-1}\pi_N V_t \pi_N) \xrightarrow[N \rightarrow \infty]{} \det(\id_H + \rho^{-1}V_t).
\]

\end{lemma}

\begin{proof}
    We know that the mapping $A \mapsto \det(\id_H + A)$ is continuous under the ``trace norm'' $\|A\|_1 := \sum_{n = 1}^\infty |\lambda_n(A)|$ \citep{lax2002functional}. Thus, to show the desired result, it suffices to show that $\|\pi_N V_t \pi_N - V_t\|_{1} \xrightarrow[N \rightarrow \infty]{} 0$. Observe that both $\pi_N V_t \pi_N$ and $V_t$ are operators of rank at most $t$, so so their difference $\pi_N V_t \pi_N - V_t$ has rank at most $2t$. Thus, we know that
    \[
    \|\pi_N V_t \pi_N - V_t\|_{1} \leq 2t\|\pi_N V_t \pi_N - V_t\|_{op} \xrightarrow[N \rightarrow \infty]{} 0,
    \]
    where the final convergence follows from Lemma~\ref{lem:op_norm_convergence}. Thus, we have shown the desired result.
\end{proof}

We now tie together all of these technical (but intuitive) results in the proof of Theorem~\ref{thm:mixture_hilbert} below.

\begin{proof}[\textbf{Proof of Theorem~\ref{thm:mixture_hilbert}}]
    Let $(\varphi_n)_{n \geq 1}$ be an orthonormal basis for $H$, and for $N \geq 1$, let $\pi_N$ denote the projection operator outlined in Section~\ref{sec:back}. Recall that $\pi_N = \pi_N^\top$. Further $H_N := \spn\{\varphi_1, \dots, \varphi_N\} \subset H$ is the image of $H$ under $\pi_N$. Since  $(S_t)_{t \geq 0}$ is an $H$-valued martingale with respect to $(\calF_t)_{t \geq 1}$, it follows that the projected process $(\pi_N S_t)_{t \geq 1}$ is an $H_N$-valued martingale with respect to  $(\calF_t)_{t \geq 0}$. Further, note that the projected variance process $(\pi_N V_t \pi^\top_N)_{t \geq 0}$ satisfies
    \[
    \pi_N V_t \pi^\top_N = \sum_{s = 1}^t (\pi_N f_s)(\pi_N f_s)^\top.
    \]
    
    Since, for any $N \geq 1$, $H_N$ is a finite-dimensional Hilbert space, it follows from Lemma~\ref{lem:mixture_hilbert_finite} that the process $(M_t^{(N)})_{t \geq 0}$ given by
    \begin{align*}
    M_t^{(N)} &:= \frac{1}{\sqrt{\wt{\det}(\id_{H_N} + \rho^{-1}\pi_N V_t \pi_N^\top)}}\exp\left\{\frac{1}{2}\left\|(\rho \id_{H_N} +\pi_N V_t \pi_N^\top)^{-1/2}\pi_N S_t\right\|_{H_N}^2\right\} \\
    &= \frac{1}{\sqrt{\det(\id_{H} + \rho^{-1}\pi_N V_t \pi_N^\top)}}\exp\left\{\frac{1}{2}\left\|(\rho \id_H +\pi_N V_t \pi_N^\top)^{-1/2}\pi_N S_t\right\|_{H}^2\right\},
    \end{align*}
    is a non-negative supermartingale with respect to $(\calF_t)_{t \geq 0}$. In the above $\id_{H_N}$ denotes the identity $\id_H$ restricted to $H_N \subset H$ and $\wt{\det}$ denotes the determinant restricted to the subspace $H_N$. The equivalence of the second and third terms above is trivial. 
    
    We now argue that for any $t \geq 1$, 
    \begin{equation}\label{eq:conv}
    \lim_{N \rightarrow \infty}M_t^{(N)} = M_t. 
    \end{equation}
    If we show this to be true, then we have, for any $t \geq 1$
    \begin{align*}
    \E\left(M_t \mid \calF_{t - 1}\right) &= \E\left(\liminf_{N \rightarrow \infty}M_t^{(N)} \mid \calF_{t - 1}\right) \\
    &\leq \liminf_{N \rightarrow \infty}\E\left(M_t^{(N)} \mid \calF_{t - 1}\right) \\
    &\leq \liminf_{N \rightarrow \infty}M_{t - 1}^{(N)} \\
    &= M_{t - 1},
    \end{align*}
    which implies $(M_t)_{t \geq 0}$ is a non-negative supermartingale with respect to $(\calF_t)_{t \geq 0}$ thus proving the result. In the above, the first inequality follows from Fatou's lemma for conditional expectations (see \citet{durrett2019probability}, for instance), and the second inequality follows from the supermartingale property.

    Lemma~\ref{lem:det_convergence} tells us that $\det(\id_H + \rho^{-1}\pi_N V_t \pi_N) \xrightarrow[N \rightarrow \infty]{} \det(\id_H + \rho^{-1} V_t)$ for all $t \geq 1$, so to show the desired convergence in~\eqref{eq:conv}, it suffices to show that \[\|(\rho \id_H + \pi_NV_t\pi_N)^{-1/2}\pi_NS_t\|_H \xrightarrow[N \rightarrow \infty]{} \|(\rho \id_H + V_t)^{-1/2}S_t\|_H \text{ for any $t$}. \] 

    %Lemma~\ref{lem:op_norm_convergence} tells us that $\|V_t - \pi_N V_t \pi_N\|_{op} \xrightarrow[N \rightarrow \infty]{} 0$, which additionally implies that \[\|(\rho \id_H + \pi_NV_t\pi_N)^{-1/2} - (\rho \id_H + V_t)^{-1/2}\|_{op} \xrightarrow[N \rightarrow \infty]{} 0.\] 

    Let $\cV_t := \rho \id_H + V_t$ and $\cV_t(N) := \rho \id_H + \pi_N V_t \pi_N$ in the following line of reason for simplicity. We trivially have
    \begin{align*}
       \left|\|\cV_t(N)^{-1/2}\pi_NS_t\|_H - \|\cV_t^{-1/2}S_t\|_H\right| &\leq \left\|\cV_t(N)^{-1/2}\pi_NS_t - \cV_t^{-1/2}S_t\right\|_H \\
       &= \left\|\cV_t(N)^{-1/2}\pi_NS_t - \cV_t(N)^{-1/2}S_t + \cV_t(N)^{-1/2}S_t - \cV_t^{-1/2}S_t\right\|_H \\
       &\leq \left\|\cV_t(N)^{-1/2}\right\|_{op}\left\|\pi_N^\perp S_t\right\|_H + \left\|\cV_t(N)^{-1/2} - \cV_t^{-1/2}\right\|_{op}\|S_t\|_H \\
       &\xrightarrow[N \rightarrow \infty]{} 0.
    \end{align*}
    as $\lim_{N \rightarrow \infty}\|\pi_N^\perp f\| = 0$ for any $f \in H$ of finite norm, and Lemma~\ref{lem:op_norm_convergence} tells us that $\|V_t - \pi_N V_t \pi_N\|_{op} \xrightarrow[N \rightarrow \infty]{} 0$, which in turn implies that $\|\cV_t(N)^{-1/2} - \cV_t^{-1/2}\|_{op} = \|(\rho \id_H + \pi_NV_t\pi_N)^{-1/2} - (\rho \id_H + V_t)^{-1/2}\|_H\xrightarrow[N \rightarrow \infty]{} 0$. Thus, we have shown the desired result.

    The second part of the claim follows from a direct application of Fact~\ref{fact:ville} and rearranging.

\end{proof}

As a final result in this appendix, we provide a proof of Corollary~\ref{cor:gram_bd}. This corollary allows for a more direct comparison of Theorem~\ref{thm:mixture_hilbert} (and thus Corollary 3.5 of \citet{abbasi2013online}) with those of \citet{chowdhury2017kernelized}. Our proof is a simple generalization Lemma 1 in the aforementioned paper to the case of arbitrary regularization parameters.

\begin{proof}[\textbf{Proof of Corollary~\ref{cor:gram_bd}}]
The first result is straightforward, and follows from the identity
\[
\det(\id_H + \rho^{-1}V_t) = \det(I_t + \rho^{-1}K_t),
\]
which we bring to attention in Section~\ref{sec:back}.

The second result follows from the following line of reasoning. Before proceeding, recall that $\Phi_t := (k(\cdot, X_1), \dots, k(\cdot, X_t))^\top$, $V_t = \Phi_t^\top \Phi_t$, $K_t = \Phi_t\Phi_t^\top$ and $S_t = \sum_{s = 1}^t\epsilon_s k(\cdot, X_s) = \Phi_t^\top \epsilon_{1:t}$.
\begin{align*}
\left\|(\rho \id_H + V_t)^{-1}S_t\right\|_H^2 &= \epsilon_{1:t}^\top \Phi_t (\rho \id_H + \Phi_t^\top \Phi_t)^{-1}\Phi_t^\top \epsilon_{1:t} \\
&= \epsilon_{1:t}^\top(\rho^{-1/2}\Phi_{t})\left(\id_H + (\rho^{-1/2}\Phi_t)^\top(\rho^{-1/2}\Phi_t)\right)^{-1}(\rho^{-1/2}\Phi_t)^\top \epsilon_{1:t} \\
&= \epsilon_{1:t}^T \rho^{-1}\Phi_t\Phi_t^\top \left(I_t + \rho^{-1}\Phi_t\Phi_t^\top\right)^{-1}\epsilon_{1:t}\\
&= \epsilon_{1:t}^\top (\rho^{-1}K_t)(I_t + \rho^{-1}K_t)^{-1}\epsilon_{1:t} \\
&= \epsilon_{1:t}^\top(I_t + \rho K_t^{-1})^{-1}\epsilon_{1:t}\\
&= \left\|(I_t + \rho K_t^{-1})^{-1/2}\epsilon_{1:t}\right\|_2^2.
\end{align*}
In the above, the second equality comes from pulling out a multiplicative factor of $\rho$ form the center operator inverse. The third inequality comes from the famed ``push through'' identity. Lastly, the second to last equality comes from observing that (a) $\rho^{-1}K_t$ and $(I_t + \rho^{-1} K_t)^{-1}$ are simultaneously diagonalizable matrices and (b) for scalars, we have the identity $(1 + a^{-1})^{-1} = a(1 + a)^{-1}$. Thus, we have shown the desired result. 
\end{proof}

\section{Technical Lemmas for Theorem~\ref{thm:reg}}
\label{app:reg}
In this appendix, we provide various technical lemmas needed for the proof of Theorem~\ref{thm:reg}. We then follow these lemmas with a full proof of Theorem~\ref{thm:reg}, which extends the sketch provided in the main body of the paper. Most of the following technical lemmas either already exist in the literature~\citep{chowdhury2017kernelized} or are extensions of what is known in the case of finite-dimensional, linear bandits~\citep{abbasi2011improved}. We nonetheless provide self-contained proofs for the sake of completeness.

\begin{lemma}
\label{lem:conf_set}
Let $(f_t)_{t \geq 1}$ be the sequence of functions defined in Algorithm~\ref{alg:UCB}, and assume Assumption~\ref{ass:regret} holds. Let $\delta \in (0, 1)$ be an arbitrary confidence parameter. Then, with probability at least $1 - \delta$, simultaneously for all $t \geq 1$, we have 
\[
\left\|(V_t + \rho \id_H)^{1/2}(f_t - f^\ast)\right\|_H \leq \sigma\sqrt{2\log\left(\frac{1}{\delta}\sqrt{\det(\id_H + \rho^{-1}V_t)}\right)} + \rho^{1/2}D,
\]
where we recall that the right hand side equals $U_t$.
\end{lemma}
\begin{proof}
First, observe that we have
\begin{align*}
f_t - f^\ast &= (\rho \id_H + V_t)^{-1}\Phi_t^\top Y_{1:t} - f^\ast \\
&= (\rho \id_H + V_t)^{-1}\Phi_t^\top (\Phi_t f^\ast + \epsilon_{1:t}) - f^\ast \\
&= (\rho \id_H + V_t)^{-1}\Phi_t^\top (\Phi_t f^\ast + \epsilon_{1:t}) - f^\ast \pm \rho(\rho \id_H + V_t)^{-1}f^\ast \\
&= (\rho \id_H + V_t)^{-1}\Phi_t^\top \epsilon_{1:t} - \rho (\rho\id_H + V_t)^{-1}f^\ast.
\end{align*}
Applying the triangle inequality to the above, we have
\begin{align*}
\left\|(\rho\id_H + V_t)^{1/2}(f_t - f^\ast)\right\|_H &\leq \left\|(\rho\id_H + V_t)^{-1/2}\Phi_t^\top \epsilon_{1:t}\right\|_H + \rho\left\|(\rho\id_H + V_t)^{-1/2}f^\ast\right\|_H \\
&\leq \sigma\sqrt{2\log\left(\frac{1}{\delta}\sqrt{\det(\id_H + \rho^{-1}V_t)}\right)} + \rho^{1/2}D.
\end{align*}
To justify the final inequality, we look at each term separately. For the first term, observe that $V_t = \rho \id_H + \sum_{s = 1}^t k(\cdot, X_t)k(\cdot, X_t)^\top$ and $S_t := \Phi_t^\top\epsilon_{1:t} = \sum_{s = 1}^t \epsilon_s k(\cdot, X_s)$. Thus, we are in the setting of Theorem~\ref{thm:mixture_hilbert}, and thus have, with probability at least $1 - \delta$, simultaneously for all $t \geq 0$,
\[
\left\|(\rho\id_H + V_t)^{-1/2}\Phi_t^\top \epsilon_{1:t}\right\|_H \leq \sigma\sqrt{2\log\left(\frac{1}{\delta}\sqrt{\det(\id_H + \rho^{-1}V_t)}\right)}.
\]
For the second term, observe that (a) $\lambda_{\min}(\rho\id_H + V_t) \geq \rho$ and (b) by Assumption~\ref{ass:regret}, we have $\|f^\ast\|_H \leq D$. Thus applying Holder's inequality, we have, deterministically
\[
\rho\left\|(\rho\id_H + V_t)^{-1/2}f^\ast\right\|_H \leq \rho\left\|(\rho\id_H + V_t)^{-1/2}\right\|_{op} \left\|f^\ast\right\|_H 
 \leq \rho^{1/2}\|f^\ast\|_H \leq \rho^{1/2}D.
\]
These together give us the desired result.

\end{proof}

The following ``elliptical potential'' lemma, abstractly, aims to control the the growth of the squared, self-normalized norm of the selected actions. We more or less port the argument from \citet{abbasi2011improved}, which provides an analogue in the linear stochastic bandit case. We just need to be mildly careful to work around the fact we are using Fredholm determinants.
\begin{lemma}
\label{lem:CS_det_bd}
For any $t \geq 1$, let $V_t$ be the covariance operator defined in Algorithm~\ref{alg:UCB}, and let $\rho > 0$ be arbitrary. We have the identity
\[
\det(\id_H + \rho^{-1}V_t) = \prod_{s = 1}^t\left(1 + \left\|(\rho \id_H + V_{s - 1})^{-1/2}k(\cdot, X_s)\right\|_H^2\right).
\]
In particular, if $\rho \geq 1 \lor L$, where $L$ is the bound outlined in Assumption~\ref{ass:kernel}, we have
\[
\sum_{s = 1}^t \left\|(\rho \id_H + V_{s - 1})^{-1/2}k(\cdot, X_s)\right\|_H^2 \leq 2\log\det(\id_H + \rho^{-1}V_t).
\]

\end{lemma}

\begin{proof}
    Let $H_t \subset H$ be the finite-dimensional Hilbert space  $H_t := \spn\{k(\cdot, X_1), \dots, k(\cdot, X_t)\}$. Let $\det_{H_t}$ denote the determinant restricted to $H_t$, i.e.\ the map that acts on a (symmetric) operator $A : H_t \rightarrow H_t$ by $\det_{H_t}(A) := \prod_{s = 1}^t \lambda_s(A)$, where $\lambda_1(A), \dots, \lambda_t(A)$ are the enumerated eigenvalues of $A$. Observe the identity
    \[
    \det(\id_H + \rho^{-1}V_t) = \det_{H_t}(\id_{H_t} + \rho^{-1}V_t),
    \]
    where we recall the determinant on the lefthand side is the Fredholm determinant, as defined in Section~\ref{sec:back}. Next, following the same line of reasoning as \citet{abbasi2011improved}, we have
    \begin{align*}
    &\det_{H_t}(\rho \id_{H_t} + V_t) \\
    &= \det_{H_t}(\rho \id_{H_t} + V_{t - 1})\det_{H_t}\left(\id_{H_t} + (\rho \id_{H_t} + V_{t - 1})^{-1/2}k(\cdot, X_t)k(\cdot, X_t)^\top (\rho \id_{H_t} + V_{t - 1})^{-1/2}\right) \\
    &=\det_{H_t}(\rho \id_{H_t} + V_{t - 1})\left(1 + \left\|(\rho \id_{H_t} + V_{t - 1})^{-1/2}k(\cdot, X_t)\right\|_H^2\right)\\
    &= \cdots \text{ (Iterating $t - 1$ more times) }\\
    &= \det_{H_t}(\rho \id_H)\prod_{s = 1}^t\left(1 + \left\|(\rho \id_{H_t} + V_{s - 1})^{-1/2}k(\cdot, X_s)\right\|_H^2\right) \\
    &= \det_{H_t}(\rho \id_H)\prod_{s = 1}^t\left(1 + \left\|(\rho \id_{H} + V_{s - 1})^{-1/2}k(\cdot, X_s)\right\|_H^2\right),
    \end{align*}
    where the last equality comes from realizing, for all $s \in [t]$, $\|(\rho \id_{H_t} + V_{s  -1})^{-1/2}k(\cdot, X_s)\|_H = \|(\rho \id_{H} + V_{s - 1})^{-1/2}k(\cdot, X_s)\|_H$. Thus, rearranging yields
    \[
    \det_{H_t}(\id_{H_t} + \rho^{-1}V_t) = \prod_{s = 1}^t\left(1 + \left\|(\rho \id_H + V_{s - 1})^{-1/2}k(\cdot, X_s)\right\|_H^2\right),
    \]
    which yields the first part of the claim. 
    
    Now, to see the second part of the claim, observe the bound $x \leq 2\log(1 + x), \forall x \in [0, 1]$. Observing that, for all $s \in [t]$, $\left\|(\rho \id_H + V_{s - 1})^{-1/2}k(\cdot, X_s)\right\|_H \leq 1$ when $\rho \geq 1 \lor L$, we have
    \begin{align*}
    \sum_{s = 1}^t \left\|(\rho \id_H + V_{s - 1})^{-1/2}k(\cdot, X_s)\right\|_H^2 &\leq 2\sum_{s = 1}^t \log\left(1 + \left\|(\rho \id_H + V_{s - 1})^{-1/2}k(\cdot, X_s)\right\|_H^2\right) \\
    &= 2\log\left(\prod_{s = 1}^t\left(1 + \left\|(\rho \id_H + V_{s - 1})^{-1/2}k(\cdot, X_s)\right\|_H^2\right)\right) \\
    &= 2\log\det(\id_H + \rho^{-1}V_t),
    \end{align*}
    proving the second part of the lemma.
\end{proof}

With the above lemmas, along with the concentration results provided by Theorem~\ref{thm:mixture_hilbert}, we can provide a full proof for Theorem~\ref{thm:reg}.

\begin{proof}[\textbf{Proof of Theorem~\ref{thm:reg}}]
We take the standard approach of (a) first bounding instantaneous regret and then (b) applying the Cauchy-Schwarz inequality to bound the aggregation of terms. To start, for any $t \in [T]$, define the ``instantaneous regret'' as $r_t := f^\ast(x^\ast) - f^\ast(X_t)$, where we recall $x^\ast := \arg\max_{x \in \calX}f^\ast(x)$. By applying Lemma~\ref{lem:conf_set}, we have with probability at least $1 - \delta$ that
\begin{align*}
r_t &= f^\ast(x^\ast) - f^\ast(X_t) \\
&\leq \wt{f}_{t}(X_t) - f^\ast(X_t) \\
&= \wt{f}_{t}(X_t) - f_{t - 1}(X_t) + f_{t - 1}(X_t) - f^\ast(X_t) \\
&= \langle \wt{f}_{t} - f_{t - 1}, k(\cdot, X_t)\rangle_H - \langle f_{t - 1} - f^\ast, k(\cdot, X_t)\rangle_H \\
&\leq \left\|(\rho \id_H + V_{t - 1})^{-1/2}k(\cdot, X_t)\right\|_H\left( \left\|(\rho \id_H + V_{t - 1})^{1/2}(\wt{f}_{t} - f_{t - 1})\right\|_H +  \left\|(\rho \id_H + V_{t - 1})^{1/2}(f_{t - 1} - f^\ast)\right\|_H \right)\\
&\leq 2U_{t - 1}\left\|(\rho \id_H + V_{t - 1})^{-1/2}k(\cdot, X_t)\right\|_H,
\end{align*}
where $\wt{f}_t$ and $f_{t - 1}$ are as in Algorithm~\ref{alg:UCB}. Note that, in the above, we apply Lemma~\ref{lem:conf_set} in obtaining the first inequality (which is the  ``optimism in the face of uncertainty'' part of the bound), and additionally in obtaining the last inequality. The second to last inequality follows from applying Cauchy-Schwarz.

With the above bound, we can apply again the Cauchy-Schwarz inequality to see
\begin{align*}
R_T &= \sum_{t = 1}^T r_t \leq \sqrt{T\sum_{t = 1}^T r_t^2} \leq U_T\sqrt{2T\sum_{t = 1}^T\left\|(\rho \id_H + V_{t - 1})^{-1/2}k(\cdot, X_t)\right\|_H^2} \\
&\leq U_T\sqrt{2T\log\det(\id_H + \rho^{-1}V_T)} \\
&= \left(\sigma\sqrt{2\log\left(\frac{1}{\delta}\sqrt{\det(\id_H + \rho^{-1}V_T)}\right)} + \rho^{1/2}D\right)\sqrt{2T\log\det(\id_H + \rho^{-1}V_T)} \\
&\leq \left(\sigma\sqrt{2\log(1/\delta)} + \sigma\sqrt{2\gamma_T(\rho)} + \rho^{1/2}D\right)\sqrt{4T\gamma_T(\rho)} \\
&= \sigma\gamma_T(\rho)\sqrt{8T} + D\sqrt{4\rho \gamma_T(\rho) T} + \sigma\sqrt{8T\log(1/\delta)}\\
&= O\left(\gamma_T(\rho)\sqrt{T} + \sqrt{\rho\gamma_T(\rho)T}\right).
\end{align*}
In the above, the second inequality follows from the second part of Lemma~\ref{lem:CS_det_bd}, the following equality follows from substituting in $U_T$, and the final inequality follows from the definition of the maximum information gain $\gamma_T(\rho)$ and the fact that $\sqrt{a + b} \leq \sqrt{a} + \sqrt{b}$ for all $a, b \geq 0$. The last, big-Oh bound is straightforward. With this, we have proven the first part of the theorem.

Now, suppose the kernel $k$ experiences $(C, \beta)$-polynomial eigendecay. Then, by Fact~\ref{fact:edecay}, we know that
\begin{align*}
\gamma_T(\rho) &\leq \left(\left(\frac{CB^2T}{\rho}\right)^{1/\beta}\log^{-1/\beta}\left(1 + \frac{LT}{\rho}\right) + 1\right)\log\left(1 + \frac{LT}{\rho}\right) \\
&= \wt{O}\left(\left(\frac{T}{\rho}\right)^{1/\beta}\right).
\end{align*}
We aim to set $\rho \asymp \left(\frac{T}{\rho}\right)^{1/\beta},$ which occurs when $\rho = O(T^{\frac{1}{1 + \beta}})$. When this happens, we have
\[
\left(\frac{T}{\rho}\right)^{1/\beta}\sqrt{T} = T^{\frac{1}{1 + \beta} + \frac{1}{2}} = T^{\frac{3 + \beta}{2 + 2\beta}}.
\]
Applying this, we have that
\begin{align*}
R_T &= O\left(\gamma_T(\rho)\sqrt{T} +\sqrt{\rho \gamma_T(\rho) T}\right)\\
&= \wt{O}\left(T^{\frac{3  + \beta}{2 + 2\beta}}\right),
\end{align*}
which, in particular, is sublinear for any $\beta > 1$. Thus, we are done.

\end{proof}

\end{document}